\documentclass[preprint,12pt,numbers]{elsarticle}
\usepackage{amssymb}
\usepackage{amsmath}
\usepackage{amsthm}
\usepackage{amsfonts}
\usepackage{booktabs}
\usepackage{xcolor}
\journal{Computer Methods in Applied Mechanics and Engineering}
\usepackage{natbib}
\usepackage{graphicx}
\usepackage{geometry}
\usepackage{url}
\usepackage{algorithmic}
\usepackage{algorithm}
\usepackage{float}
\usepackage[pdfencoding=auto]{hyperref}
\usepackage{cleveref}
\usepackage{siunitx} 
\usepackage{caption}
\usepackage{subcaption}
\usepackage[percent]{overpic}
\usepackage{multirow}
\usepackage{makecell}
\usepackage[table]{xcolor}
\usepackage{tikz}
\usepackage{bm} 

\usepackage[percent]{overpic}

\newcommand{\sci}[2]{\ensuremath{#1\times10^{#2}}}

\newtheorem{theorem}{Theorem}

\begin{document}
\begin{frontmatter}

\title{SetONet: A Set-Based Operator Network for Solving PDEs with Variable-Input Sampling} 

\author[1]{Stepan Tretiakov\corref{cor1}}
\ead{stepan@utexas.edu}

\author[2]{Xingjian Li}
\ead{xingjian.li@austin.utexas.edu}

\author[1,2]{Krishna Kumar}
\ead{krishnak@utexas.edu}

\cortext[cor1]{Corresponding author.}

\affiliation[1]{organization={Maseeh Department of Civil, Architectural and Environmental Engineering},
            addressline={University of Texas at Austin}, 
            city={Austin},
            postcode={78712}, 
            state={TX},
            country={USA}}
\affiliation[2]{organization={Oden Institute for Computational Engineering and Sciences},
            addressline={University of Texas at Austin}, 
            city={Austin},
            postcode={78712}, 
            state={TX},
            country={USA}}

\begin{abstract}
Many neural-operator surrogates for PDEs inherit a restrictive assumption from DeepONet-style formulations: the input function must be sampled at a fixed, ordered set of sensors. This assumption limits applicability to problems with variable sensor layouts, missing data, point sources, and sample-based representations of densities. We propose SetONet, which addresses this gap by recasting the operator input as an unordered set of coordinate–value observations and encoding it with permutation-invariant aggregation inside a standard branch–trunk operator network while preserving the DeepONet synthesis mechanism and lightweight end-to-end training. A structured variant, SetONet-Key, aggregates sensor information through learnable query tokens and a position-only key pathway, thereby decoupling sampling geometry from sensor values. The method is assessed on four classical operator-learning benchmarks under fixed layouts, variable layouts, and evaluation-time sensor drop-off, and on four problems with inherently unstructured point-cloud inputs, including heat conduction with multiple point sources, advection–diffusion, phase-screen diffraction, and optimal transport problems. In parameter-matched studies, SetONet-Key achieves lower error than the DeepONet baseline on fixed-sensor benchmarks and remains reliable when layouts vary or sensors are dropped at evaluation. Comparisons across pooling rules show that attention-based aggregation is typically more robust than mean or sum pooling. On the point-cloud problems, SetONet operates directly on the native input representation, without rasterization or multi-stage preprocessing, and outperforms the larger VIDON baseline.
\end{abstract}            


\end{frontmatter}


\section{Introduction}
\label{sec:intro}

Many families of partial differential equations (PDEs) induce a \emph{solution operator} that maps an input function (e.g., a forcing term, a coefficient field, or boundary data) to the corresponding solution field.
A neural operator~\cite{kovachki2023neural, azizzadenesheli2024neural} is a deep neural network that seeks to learn such a mapping $\mathcal{T}:\mathcal{G}\to\mathcal{H}$ between (typically infinite-dimensional) spaces of input and output functions.
Learning $\mathcal{T}$ from data enables fast surrogate models that amortize expensive numerical solves and can be evaluated at arbitrary query locations, making neural operators attractive for scientific and engineering modeling~\cite{oommen2022learning,kurth2023fourcastnet,wang2024deep}.

A broad ecosystem of operator-learning architectures has emerged.
Deep Operator Networks (DeepONets)~\cite{lu2021learning} follow a branch--trunk template motivated by universal approximation results for nonlinear operators~\cite{chen1995universal}: the branch network maps the input function to coefficients, the trunk network maps the query location to basis functions, and the solution is synthesized by combining the two.
Other families parameterize operators via learned integral transforms, including the Fourier Neural Operator (FNO)~\cite{li2021fourier,kovachki2021universal}, Wavelet Neural Operator (WNO)~\cite{tripura2023wavelet}, Graph Kernel Network (GKN)~\cite{anandkumar2020neural}, Laplace Neural Operator (LNO)~\cite{cao2024laplace}, and related spectral and convolutional variants~\cite{fanaskov2023spectral,raonic2023convolutional,wei2023super,shih2025transformers}.
More recently, transformer-based models have adapted self-attention to operator learning, including the In-Context Operator Network~\cite{yang2023context}, the Operator Transformer (OFormer)~\cite{li2022transformer}, and the General Neural Operator Transformer (GNOT)~\cite{hao2023gnot}, often improving empirical performance while increasing computational and architectural complexity.

Despite this progress, a central practical challenge remains: \emph{how input functions are observed}.
In many applications, the input is not naturally available as values on a fixed grid. Instead, it may be measured by sensors at varying locations, contain missing readings, arrive as unordered samples, or be intrinsically defined as a discrete set (e.g., point sources or particle samples).
In these settings, the most faithful representation of the available information is an unordered set of location--value pairs, and a neural operator should (i) be \emph{permutation invariant} with respect to the ordering of observations and (ii) accommodate \emph{variable cardinality} without resorting to ad hoc imputation or a prescribed discretization.

This mismatch is particularly explicit in standard DeepONet formulations~\cite{lu2019deeponet, lu2021learning}.
DeepONet's branch network encodes the input function into a fixed-length vector tied to predetermined sensor locations. Consequently, varying sensor layouts or inherently unstructured inputs fall outside the model's native interface and require additional machinery in practice~\cite{prasthofer2022variable, karumuri2024efficient}.
At the same time, retaining the conceptual and computational simplicity of the branch--trunk construction is desirable, especially in PDE settings where lightweight models and stable end-to-end training are often preferred.

We introduce \emph{SetONet}, a set-based operator network that removes the fixed-sensor assumption while preserving the DeepONet synthesis mechanism.
The key idea is to represent the input function $g$ as an unordered set of samples $\{(\boldsymbol{x}_i, g(\boldsymbol{x}_i))\}_{i=1}^M$ and to replace the standard DeepONet branch network with a novel permutation-invariant \emph{set encoder} extended from Deep Sets~\cite{zaheer2018deepsets}.
As in DeepONet, the output function is expressed in a finite learned basis; SetONet is designed so that the set encoder processes input functions as set representations and produces the corresponding operator coefficients aligned with the learned basis functions, enabling an end-to-end mapping from variable-size sample sets to solution fields.
Building on classical aggregation mechanisms, in this work we primarily adopt a lightweight token--sensor aggregation that uses learnable query tokens and position-derived keys. This design, which we refer to as \emph{SetONet-Key}, introduces a position-only key pathway while keeping the value pathway lightweight, enabling geometry-aware aggregation that decouples sensor geometry from sensor values and improves expressivity. See \cref{fig:setonet_architecture}.

A key design goal is to broaden admissible input representations without sacrificing the efficiency of DeepONet-style models.
SetONet retains the standard branch--trunk synthesis structure and is trained end-to-end, avoiding the multi-stage pipelines used by some basis-projection~\cite{ingebrand2025basis} or implicit-representation approaches~\cite{bahmani2024resolutionindependentneuraloperator}.
Through careful architectural choices, we maintain parameter counts and training cost comparable to a standard DeepONet baseline.
When inputs are sampled on fixed grids, SetONet recovers DeepONet-style behavior, making it a drop-in replacement that substantially broadens the class of admissible inputs.

We summarize the main contributions of this paper as follows:
\begin{itemize}
    \item \textbf{A set-based operator network for variable and unstructured input sampling.}
    We introduce SetONet, which treats input functions as unordered sets of location--value pairs and thereby accommodates variable, incomplete, or irregular observations without rasterization or interpolation onto a fixed grid.
    We compare several permutation-invariant aggregation strategies (mean, sum, and cross-attention pooling) and find attention-based pooling to be broadly robust, particularly under variable-sensor and evaluation-time drop-off settings where naive pooling can degrade substantially.

    \item \textbf{A geometry-aware token--sensor aggregation mechanism for set-based operator learning.}
    We propose SetONet-Key, a variant that introduces a separate position-only key pathway and learnable query tokens into the branch network, aggregating sensor information through a non-normalized, geometry-weighted token--sensor aggregation mechanism. By decoupling the encoding of sensor geometry from sensor values, SetONet-Key achieves the best or near-best accuracy among all tested variants on most benchmarks, with especially pronounced gains on fixed-sensor and point-cloud tasks.

    \item \textbf{A universal approximation result for SetONet-Key.}
    We prove that SetONet-Key is a universal approximator for continuous operators on the classical fixed-sensor operator-learning subclass (\cref{thm:setonet_key_uat}), confirming that the separate key--value pathway design does not restrict approximation capacity.

    \item \textbf{Numerical validation on classical benchmarks and problems with inherently unstructured inputs.}
    We evaluate SetONet on four classical operator learning benchmarks under fixed, variable, and sensor drop-off configurations, as well as on four problems with unstructured point-cloud inputs where standard DeepONet is not directly applicable.
    The latter group includes two new operator-learning benchmarks introduced in this work: a phase-screen diffraction problem whose output is a global interference pattern rather than an additive superposition of individual source contributions, and an optimal transport problem in which the input density is represented by discrete samples rather than grid evaluations.

    \item \textbf{An efficient implementation and open-source code.}
    SetONet preserves DeepONet's lightweight branch--trunk template at comparable training cost, while substantially broadening the class of admissible inputs.
    In parameter-matched fixed-sensor studies, SetONet-Key achieves lower error than the DeepONet baseline.
    SetONet-Key also outperforms VIDON on all benchmarks except Optimal Transport, where it matches within the reported uncertainty, while using substantially fewer parameters.
    A complete open-source implementation of all variants, baselines, and benchmarks is available at \url{https://github.com/st-ep/SetONet}.
    
\end{itemize}

The rest of the paper is organized as follows. Section~\ref{sec:related_work} reviews related work in operator learning and approaches for handling variations in inputs. Section~\ref{sec:setonet_methodology} introduces the SetONet framework and details its architecture. Section~\ref{sec:setonet_key_uat} establishes a universal approximation result for SetONet-Key on the classical fixed-sensor subclass. Section~\ref{sec:implementation} discusses implementation details. Section~\ref{sec:results} presents experimental results on benchmark PDE problems, comparing SetONet with standard DeepONet under different input sampling conditions, as well as generalizations to problems not suitable for DeepONet. Finally, Section~\ref{sec:conclusion} summarizes our findings and outlines directions for future work.

\begin{figure}[!htbp]
    \centering
    \includegraphics[width=\textwidth]{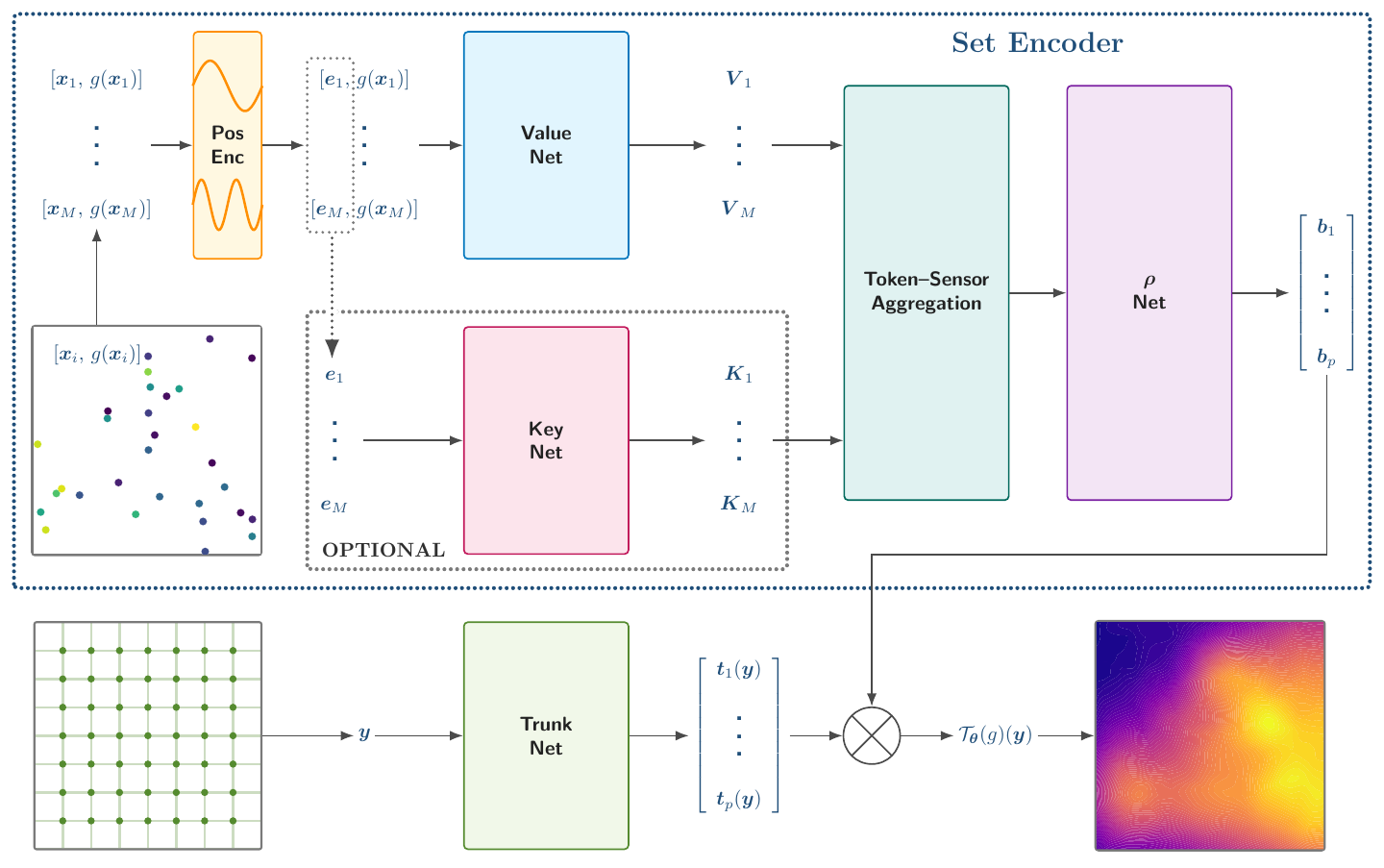}
    \caption{Schematic of the proposed SetONet architecture. The set encoder (top)
    processes the input function as an unordered set of $M$ location--value pairs
    $\{(\boldsymbol{x}_i, g(\boldsymbol{x}_i))\}_{i=1}^M$. Each sensor location
    is mapped to a positional embedding
    $\boldsymbol{e}_i = \mathrm{PE}(\boldsymbol{x}_i)$, and a Value Net
    $v_{\boldsymbol{\theta}}$ produces sensor-level embeddings $\boldsymbol{V}_i$
    (see \cref{sec:setonet_architecture} for variant-specific details). These are
    combined via permutation-invariant token--sensor aggregation, and a readout
    network $\rho$ produces basis coefficients
    $\boldsymbol{b} = [\boldsymbol{b}_1, \dots, \boldsymbol{b}_p]$. The trunk
    network (bottom) maps each query location $\boldsymbol{y}$ to basis vectors
    $\boldsymbol{t}(\boldsymbol{y}) = [\boldsymbol{t}_1(\boldsymbol{y}), \dots,
    \boldsymbol{t}_p(\boldsymbol{y})]$, and the output
    $\mathcal{T}_{\boldsymbol{\theta}}(g)(\boldsymbol{y})$ is computed via
    \cref{eq:setonet_output_simplified}. In the SetONet-Key variant, an optional
    position-only Key Net $k_{\boldsymbol{\theta}}$ maps $\boldsymbol{e}_i$ to key
    features $\boldsymbol{K}_i$, decoupling sensor geometry from sensor values in
    the aggregation step.}
\label{fig:setonet_architecture} %
\end{figure}

\section{Related Work}
\label{sec:related_work}



Classic neural operator architectures such as DeepONet~\cite{lu2019deeponet,lu2021learning} and FNO~\cite{li2021fourier,kovachki2021universal} assume inputs are observed on fixed discretizations of the domain, motivating recent work that aims to relax this constraint.

A direct approach involves designing architectures specifically for variable inputs of different sizes. 
The Variable-Input Deep Operator Network (VIDON)~\cite{prasthofer2022variable} explicitly addresses this challenge. 
VIDON processes sets of input coordinate-value pairs $\{ (\boldsymbol{x_i}, g(\boldsymbol{x_i})) \}_{i=1}^m$, where the number of points $m$ can vary. 
It employs separate MLPs to encode coordinates and values, combines them, and then uses a multi-head attention-like mechanism inspired by Transformers to produce a fixed-size representation of the input function, irrespective of $m$. 
The resulting architecture is permutation invariant, with approximation bounds established under uniform sensor sampling~\cite{prasthofer2022variable}. We note, however, that these bounds are asymptotic in the sensor count and hold in expectation, whereas the classical DeepONet result provides deterministic guarantees at a fixed discretization.
Empirical results show its robustness across various irregular sampling scenarios, including missing data and random sensor placements, at the cost of increased computational overhead.

Similar attention-based methods have been developed in recent years.
Integrating attention into the DeepONet branch network allows the model to dynamically weight input sensor readings based on their importance or location~\cite{zhu2024fast}. 
This dynamic weighting can adapt to varying numbers and positions of sensors. Specific implementations include using attention layers within the Branch Network for tasks like nonlinear fiber transmission simulation~\cite{zhu2024fast}. 
Furthermore, concepts like orthogonal attention have been proposed to build neural operators robust to irregular geometries~\cite{xiao2023improved}. 
Architectures like the Latent Neural Operator utilize a physics-informed cross-attention mechanism, enabling operation in a latent space and prediction at arbitrary locations, implicitly handling variable inputs~\cite{wang2024latent}. 
Other attention-based operators like ONO~\cite{xiao2023improved}, PIANO~\cite{li2024enhancing}, CoDA-NO~\cite{rahman2024pretraining}, and Kernel-Coupled Attention~\cite{kissas2022learning} further underscore the utility of attention for flexible input handling in operator learning.

In a more recent work, the Resolution Independent Neural Operator (RINO)~\cite{bahmani2024resolutionindependentneuraloperator} proposes an alternative approach that also handles input functions sampled at different coordinates. Instead of directly using function values, a set of neural network-parameterized basis functions is learned to approximate the input function space. The DeepONet architecture then takes the corresponding projection coefficients, obtained via implicit neural representations as input, enabling it to process functions with varying resolutions. Similar ideas are also explored in \cite{ingebrand2025basis} with a greater deviation from the standard DeepONet framework.

For certain classes of functions—such as density fields that may not admit closed-form expressions—additional preprocessing is often required to make the data compatible with the DeepONet framework. For example, \cite{cheng2021usingneuralnetworkssolve} first casts the input distribution as images (i.e., rasterize them onto a fixed grid) and then processes them with a convolutional neural network (CNN). Similar CNN-based methods are also explored in \cite{shan2020study}.

Graph Neural Networks (GNNs) offer a natural framework for processing data defined on irregular domains or meshes, making them suitable for variable input sampling~\cite{li2020neural}. 
Input samples can be represented as nodes in a graph, with features corresponding to function values and edges representing spatial relationships (e.g., proximity). 
GNNs integrated into the DeepONet branch network can learn features by aggregating information from neighboring sample points or mesh nodes via message passing, inherently handling variations in the number and arrangement of points. Hybrid models combining GNNs and DeepONet have demonstrated success. 
For example, GraphDeepONet uses a GNN for temporal evolution on irregular grids while DeepONet captures spatial characteristics~\cite{li2020neural, cho2024learning}. 
MeshGraphNet has shown promise for flow fields on irregular grids~\cite{fortunato2022multiscale,peyvan2025fusion}, and Point-DeepONet integrates PointNet (a GNN for point clouds) to process 3D geometries directly~\cite{cho2024learning}. 
These works establish GNNs as a viable method for equipping DeepONet with variable and irregular input capabilities.

Finally, alternative architectural modifications have been explored. 
Fusion DeepONet~\cite{peyvan2025fusion} employs neural field concepts and modifies the branch-trunk interaction to improve generalization on irregular grids, outperforming standard DeepONet in such settings. 
Its ability to handle arbitrary grids suggests flexibility with respect to variable input locations. 
As mentioned earlier, Latent Neural Operators~\cite{wang2024latent} operate in a learned latent space, decoupling the operator learning from the physical input grid and allowing decoding at arbitrary points, which inherently supports variable input structures. 
These examples show that fundamental changes to the DeepONet structure or operating paradigm can also yield architectures capable of handling non-standard input sampling.

Collectively, these approaches demonstrate significant progress in addressing the mismatch between operator learning models and irregular observation of input functions. At the same time, many of these methods rely on substantially more complex architectures involving stacked self-attention layers and graph message passing, or require multi-stage training, thereby increasing parameter counts and computational cost relative to the end-to-end trained DeepONet- or FNO-style models which are often more favorable in practice.

\section{SetONet for PDE Operator Learning}
\label{sec:setonet_methodology}
As discussed in \cref{sec:intro}, the standard DeepONet branch network~\cite{lu2019deeponet} requires the input function $g\in\mathcal{G}$ to be evaluated at a fixed set of sensor locations, precluding direct application to problems with variable, missing, or inherently unstructured observations.
We introduce \emph{Set Operator Network} (SetONet), which retains the DeepONet operator-learning template but replaces the fixed-sensor branch with a permutation-invariant set encoder extended from Deep Sets~\cite{zaheer2018deepsets}.
SetONet represents each input function as an unordered set of location--value pairs
$\big\{(\boldsymbol{x}_i, g(\boldsymbol{x}_i))\big\}_{i=1}^M$,
where $M$ may vary across different input functions, and processes this set through an encoder that is both \emph{spatially aware} (it operates on coordinates and values jointly) and \emph{permutation invariant}. Formally, for any set $X = \{\boldsymbol{x}_1, \boldsymbol{x}_2, \dots, \boldsymbol{x}_M\} \subset \mathbb{R}^d$ and any permutation $\pi$ of indices $ \{1, \dots, M\} $, a function $ f: 2^{\mathbb{R}^d} \rightarrow \mathbb{R} $ is permutation invariant if:
\begin{equation}
    f(\{\boldsymbol{x}_1, \boldsymbol{x}_2, \dots, \boldsymbol{x}_M\}) = f(\{\boldsymbol{x}_{\pi(1)}, \boldsymbol{x}_{\pi(2)}, \dots, \boldsymbol{x}_{\pi(M)}\}).
\end{equation}
This property is crucial for unstructured data presented as set-valued inputs, and our architecture achieves it by applying shared transformations independently to each sensor, followed by permutation-invariant aggregation and a final readout to branch coefficients. The resulting set encoder serves as the branch network, while the output space $\mathcal{H}$ is approximated, as in DeepONet, by learned basis functions produced by a trunk network and evaluated at arbitrary query locations $\boldsymbol{y}$. The specific architectural variants and their design rationale are presented in \cref{sec:setonet_architecture}.

\subsection{SetONet Architecture}
\label{sec:setonet_architecture}

SetONet follows a branch--trunk decomposition: the branch network encodes the input function, while the trunk network maps query locations to basis functions; see \cref{fig:setonet_architecture}. We first introduce \emph{SetONet-Key}, a structured branch variant that employs token--sensor aggregation with a separate position-only key pathway, treating sensor locations and sensor values through distinct networks. This formulation naturally motivates the question of whether separate pathways are in fact necessary; we therefore also consider a simplified \emph{SetONet} variant that uses a shared network applied to joint position--value features, while sharing the same trunk and output synthesis rule. We note that \cite{chiu2024deeposets} proposed a related shared-network set-encoding architecture for in-context metalearning; the simplified SetONet variant differs in that it embeds the set encoder directly within the DeepONet branch network for PDE operator learning.

\subsubsection{Branch Network}
\label{sec:branch_net}

The branch network is a permutation-invariant set encoder that maps an unordered sensor set
$\{(\boldsymbol{x}_i,\boldsymbol{u}_i)\}_{i=1}^M$, with
$\boldsymbol{x}_i \in \mathbb{R}^{d_x}$ and
$\boldsymbol{u}_i = g(\boldsymbol{x}_i) \in \mathbb{R}^{d_u}$,
to $p$ branch-coefficient vectors
$\{\boldsymbol{b}_k\}_{k=1}^{p}$, where
$\boldsymbol{b}_k \in \mathbb{R}^{d_{\text{out}}}$.
These coefficient vectors are then paired with the learned trunk basis vectors via the synthesis rule in \cref{eq:setonet_output_simplified}. Both variants aggregate sensor information through learned token--sensor interactions. SetONet-Key uses a separate position-only key pathway together with a value pathway to obtain a geometry-aware aggregation mechanism; the simplified variant instead applies attention pooling directly to joint position--value embeddings. See \cref{fig:setonet_architecture}.

\paragraph{Input Processing and Positional Encoding (PE)}
To incorporate spatial information, each sensor location $\boldsymbol{x}_i$ is mapped to a Fourier-feature positional embedding~\cite{rahimi2007random,tancik2020fourier}, yielding
$\boldsymbol{e}_i = \mathrm{PE}(\boldsymbol{x}_i) \in \mathbb{R}^{d_{\mathrm{PE}}}$.
This encoding captures multiscale spatial structure and is used throughout the branch. The two branch variants differ in how the encoded locations $\boldsymbol{e}_i$ and sensor values $\boldsymbol{u}_i$ are mapped into sensor-level features, and in how those features are aggregated by the query tokens.

\paragraph{Common token--sensor aggregation form}
The pooling step lies at the heart of the set encoder, where sensor-wise embeddings \(\boldsymbol V_i\) are combined into a set-level representation \(\boldsymbol V_{\mathrm{agg}}\), with common approaches initially proposed in~\cite{zaheer2018deepsets} including
\[
\boldsymbol{V}_{\mathrm{agg}} = \textstyle\sum_{i=1}^{M} \boldsymbol{V}_i, \quad \text{or} \quad
\boldsymbol{V}_{\mathrm{agg}} = \frac{1}{M}\textstyle\sum_{i=1}^{M} \boldsymbol{V}_i;\qquad
\]
Although sum and mean pooling are computationally simple, they can compress the sensor set too aggressively for PDE operator learning, where high accuracy often requires more expressive token--sensor interactions. To this end a more general aggregation method is needed. 
To express the token-based aggregation rules of SetONet-Key and simplified SetONet in a unified form, we introduce the following token--sensor aggregation operator. Let
$\mathbf{Q} \in \mathbb{R}^{n_{\mathrm{pool}} \times d_k}$ denote a matrix of learnable query tokens, and let
$\mathbf{K} \in \mathbb{R}^{M \times d_k}$ and
$\mathbf{V} \in \mathbb{R}^{M \times d_v}$ denote sensor-side key and value features, respectively. We define the score matrix
\[
\mathbf{S} = \frac{\mathbf{Q}\mathbf{K}^\top}{\sqrt{d_k}}
\in \mathbb{R}^{n_{\mathrm{pool}} \times M},
\]
and the token--sensor aggregation operator
\[
\operatorname{TSA}_{\mathcal{A}}(\mathbf{Q},\mathbf{K},\mathbf{V})
=
\mathcal{A}(\mathbf{S})\,\mathbf{V}
\in \mathbb{R}^{n_{\mathrm{pool}} \times d_v},
\]
where $\mathcal{A}(\mathbf{S}) \in \mathbb{R}^{n_{\mathrm{pool}} \times M}$ is a row-wise token--sensor mixing matrix. 
Single-head cross-attention is recovered when $\mathcal{A}(\mathbf{S})$ is given by row-wise softmax.
 More generally, $\mathcal{A}$ may be viewed as an abstract row-wise map from token--sensor scores to mixing coefficients, provided it remains equivariant with respect to permutations of the sensor index. This allows the same notation to cover other token--sensor aggregation rules while preserving permutation invariance.

 \paragraph{(i) SetONet-Key: Key Net and token--sensor aggregation}
SetONet-Key introduces separate sensor-side key and value pathways. A position-only Key Net
$k_{\boldsymbol{\theta}}:\mathbb{R}^{d_{\mathrm{PE}}}\to\mathbb{R}^{d_k}$ maps the encoded sensor locations to key features
$\boldsymbol{K}_i = k_{\boldsymbol{\theta}}(\boldsymbol{e}_i)\in\mathbb{R}^{d_k}$,
while a Value Net
$v_{\boldsymbol{\theta}}:\mathbb{R}^{d_u}\to\mathbb{R}^{d_v}$ maps the sensor values to value features
$\boldsymbol{V}_i = v_{\boldsymbol{\theta}}(\boldsymbol{u}_i)\in\mathbb{R}^{d_v}$, for
$i=1,\dots,M$.
Let
$\mathbf{K}=[\boldsymbol{K}_1, \dots, \boldsymbol{K}_M]^\top \in\mathbb{R}^{M\times d_k}$ and
$\mathbf{V}=[\boldsymbol{V}_1, \dots, \boldsymbol{V}_M]^\top \in\mathbb{R}^{M\times d_v}$.

SetONet-Key is obtained from the common token--sensor aggregation form by taking
$\mathbf{Q}=[\boldsymbol{q}_1;\dots;\boldsymbol{q}_{n_{\mathrm{pool}}}]\in\mathbb{R}^{n_{\mathrm{pool}}\times d_k}$
to be a set of learnable query tokens, together with the position-only keys $\mathbf{K}$ and value features $\mathbf{V}$ defined above. This yields
\[
\mathbf{P}
=
\operatorname{TSA}_{\mathcal{A}_{\mathrm{key}}}(\mathbf{Q},\mathbf{K},\mathbf{V})
\in \mathbb{R}^{n_{\mathrm{pool}}\times d_v},
\]
whose rows define token-wise latent summaries. The main structural property of SetONet-Key is that it decouples token--sensor mixing from the sensor value features through a separate position-only key pathway; the simplified variant introduced below uses the same embeddings for both roles.
Additionally, in SetONet-Key we do not use row-wise softmax normalization, since it can impose aggressive normalization and cause the aggregation to concentrate on only a few value embeddings, as discussed in~\cite{zhai2023stabilizing}.
Instead, we use alternatives such as softplus or tanh as row-wise transformations of the score matrix. This non-normalized setting is also the one covered by the universal approximation result in \cref{sec:setonet_key_uat}. This allows the geometry-dependent mixing to be specified independently of the value pathway. We apply linear scaling to ensure that the aggregated results do not explode as the number of sensors increases.

\paragraph{(ii) Simplified SetONet: Value Net and set pooling}
In the standard SetONet branch, a shared Value Net $v_{\boldsymbol{\theta}}$ is applied independently to each concatenated sensor feature $[\boldsymbol{e}_i,\boldsymbol{u}_i]$, producing
$\boldsymbol{V}_i = v_{\boldsymbol{\theta}}([\boldsymbol{e}_i,\boldsymbol{u}_i]) \in \mathbb{R}^{d_v}$ for
$i=1,\dots,M$.
Let $\mathbf{V}=[\boldsymbol{V}_1, \dots, \boldsymbol{V}_M]^\top \in\mathbb{R}^{M\times d_v}$ denote the matrix of sensor embeddings. Aggregation is then performed over these learned value embeddings, rather than directly over the raw position--value pairs. 

The embeddings are pooled over the sensor index using a symmetric operator. For the attention-based variant, referred to as SetONet (Attention), the encoder is obtained from the common token--sensor aggregation form by taking
$\mathbf{Q}\in\mathbb{R}^{n_{\mathrm{pool}}\times d_v}$ to be a set of learnable pooling tokens, using the same sensor embedding matrix as both keys and values, and choosing $\mathcal{A}$ to be the row-wise softmax. This gives
\[
\mathbf{P}
=
\operatorname{TSA}_{\mathrm{softmax}}(\mathbf{Q},\mathbf{V},\mathbf{V})
\in \mathbb{R}^{n_{\mathrm{pool}}\times d_v},
\qquad
\boldsymbol{V}_{\mathrm{agg}}
=
\operatorname{vec}(\mathbf{P})
\in \mathbb{R}^{n_{\mathrm{pool}} \cdot d_v}.
\]
Here $\operatorname{vec}(\cdot)$ flattens the pooled token matrix into a vector. Thus, the attention-pooling SetONet variant reduces to a standard cross-attention aggregation with learnable weights applied to joint position--value embeddings. We note that in all experiments the SetONet (Attention) variant uses $n_{\mathrm{pool}}=1$, yielding a single pooled token.

\paragraph{Readout to branch coefficients}
In both variants, the aggregated latent features are mapped to the branch coefficients
$\{\boldsymbol{b}_k\}_{k=1}^{p}$, where $\boldsymbol{b}_k \in \mathbb{R}^{d_{\text{out}}}$. Equivalently, the full branch output may be viewed as the concatenated vector $\boldsymbol{b}\in\mathbb{R}^{p d_{\text{out}}}$.
For SetONet-Key, since flattening the aggregation results $\mathbf{P}$ can lead to a significant increase in size and computation for the readout, we opt for an efficient parameterization that balances accuracy and cost. We define $\rho_{\mathrm{tok}}:\mathbb{R}^{d_v}\to\mathbb{R}^{d_{\text{out}}}$ to be a small-sized MLP applied row-wise to the token summaries, and let $W:\mathbb{R}^{n_{\mathrm{pool}}}\to\mathbb{R}^{p}$ denote a linear projection applied across the token dimension. This yields
\[
\boldsymbol{b}=\operatorname{vec}\!\bigl(W\rho_{\mathrm{tok}}(\mathbf{P})\bigr)\in\mathbb{R}^{p d_{\text{out}}},
\]
where $\rho_{\mathrm{tok}}$ is applied row-wise and $\boldsymbol{b}$ is reshaped into $\{\boldsymbol{b}_k\}_{k=1}^{p}$ with $\boldsymbol{b}_k \in \mathbb{R}^{d_{\text{out}}}$.
For SetONet (Attention), the set-level representation $\boldsymbol{V}_{\mathrm{agg}}$ is mapped by a readout MLP
\[
\rho:\mathbb{R}^{n_{\mathrm{pool}}\cdot d_v}\to\mathbb{R}^{p d_{\text{out}}},
\]
whose output is reshaped into $\{\boldsymbol{b}_k\}_{k=1}^{p}$.
In either case, the branch outputs the coefficient vectors $\{\boldsymbol{b}_k\}_{k=1}^{p}$, which are then combined with the trunk basis functions through the common synthesis rule in \cref{eq:setonet_output_simplified}.

\paragraph{Variable-cardinality behavior}
A direct consequence of the set-based branch formulation is that the branch map is defined for arbitrary sensor cardinality $M$: changing $M$ changes only the number of elements in the input set, not the architecture of the encoder. Figure~\ref{fig:sensor_ablation} illustrates this property on the Darcy 1D benchmark. All models are trained at $M=300$ and evaluated at different sensor counts without retraining. The stable behavior of the SetONet variants, especially SetONet-Key, indicates that the learned branch representation is set-based and not tied to a single nominal discretization, but remains meaningful as the sampling density changes. The sum-pooling variant provides a useful contrast: because its aggregation is unnormalized, the scale of the pooled representation changes directly with the sensor cardinality, making it substantially more sensitive to shifts in $M$ away from the training value.

\begin{table}[htbp]
\centering
\caption{Number of trainable parameters comparing SetONet variants, DeepONet, and VIDON on two representative benchmarks. `---' denotes not applicable. }
\label{tab:param_counts}
\fontsize{10pt}{14.4pt}\selectfont
\setlength{\tabcolsep}{5pt}
\renewcommand{\arraystretch}{1.12}
\begin{tabular}{@{}lccccccc@{}}
\toprule
\multirow{2}{*}{\textbf{Benchmark}} 
& \multirow{2}{*}{\textbf{SetONet-Key}}
& \multicolumn{3}{c}{\textbf{SetONet}} 
& \multirow{2}{*}{\textbf{DeepONet}} 
& \multirow{2}{*}{\textbf{VIDON}} \\
\cmidrule(lr){3-5}
& & \textbf{Attention} & \textbf{Mean} & \textbf{Sum} &  &  \\
\midrule
Darcy 1D     & 207{,}842 & 255{,}021 & 250{,}765 & 250{,}765 & 281{,}792 & 695{,}893 \\
Diffraction  & 301{,}956 & 367{,}810 & 363{,}554 & 363{,}554 & \textemdash & 811{,}622 \\
\bottomrule
\end{tabular}
\end{table}

\begin{figure}[ht]
    \centering
    \includegraphics[width=1.0\textwidth]{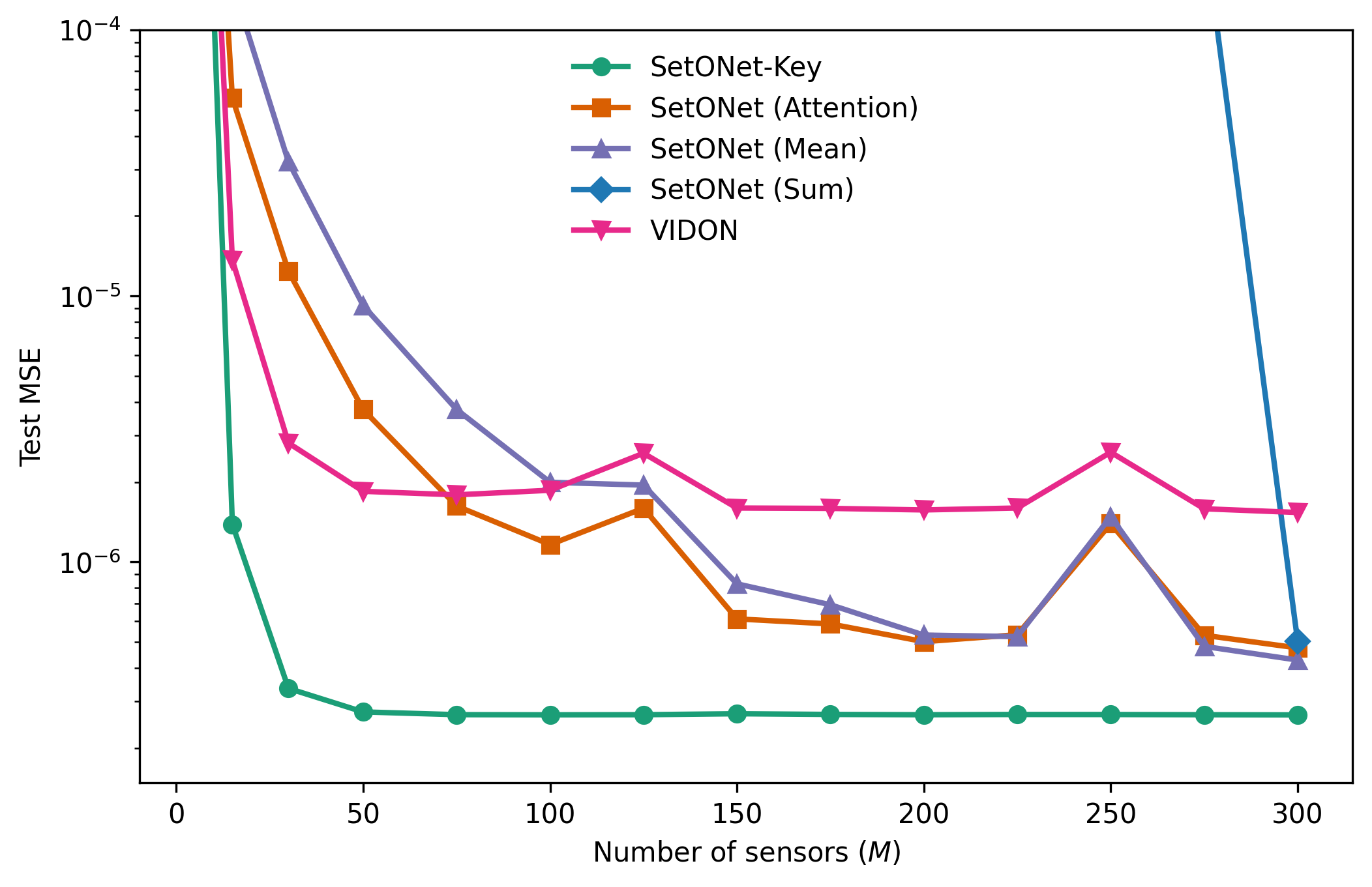}
    \caption{Test MSE as a function of the number of equally spaced sensors $M$ on the Darcy 1D benchmark. All models are trained with $M{=}300$ sensors and evaluated at varying sensor counts without retraining. The stable behavior indicates that the learned representation of a function is set-based and remains consistent when the sample size changes. We also note that the sum pooling from~\cite{zaheer2018deepsets} explodes as the number of sensors changes, highlighting the need for our improved set encoder designs.}
    \label{fig:sensor_ablation}
\end{figure}

\subsubsection{Trunk Network}

The trunk is a standard MLP that takes a query location $\boldsymbol{y}\in\mathbb{R}^{d_y}$ and returns $p$ basis vectors
\[
\boldsymbol{t}(\boldsymbol{y})=\big[\boldsymbol{t}_1(\boldsymbol{y}),\dots,\boldsymbol{t}_p(\boldsymbol{y})\big],\qquad
\boldsymbol{t}_k(\boldsymbol{y})\in\mathbb{R}^{d_{\text{out}}}.
\]
Its role mirrors the trunk in DeepONet; architectural hyperparameters (depth, width, activations) are task-dependent.

\subsubsection{Output Computation}
The final prediction at a query point $\boldsymbol{y}$ is obtained by combining the branch coefficients $\boldsymbol{b}$ with the trunk basis functions $\boldsymbol{t}(\boldsymbol{y})$, following the standard DeepONet formulation. Specifically,
\begin{equation}
   \mathcal {T}_{\boldsymbol{\theta}}(g)(\boldsymbol{y}) 
    = \sum_{k=1}^{p} \boldsymbol{b}_k \odot \boldsymbol{t}_k(\boldsymbol{y}) + \boldsymbol{b}_0,
    \label{eq:setonet_output_simplified}
\end{equation}
where $\boldsymbol{b}_k, \boldsymbol{t}_k(\boldsymbol{y}) \in \mathbb{R}^{d_{\text{out}}}$, $\odot$ denotes element-wise multiplication, and $\boldsymbol{b}_0 \in \mathbb{R}^{d_{\text{out}}}$ is a learnable bias. This preserves the DeepONet synthesis mechanism while leveraging the permutation-invariant, spatially aware branch representation.

\section{Universal approximation of SetONet-Key}
\label{sec:setonet_key_uat}

We prove that SetONet-Key is a universal approximator on the classical fixed-sensor operator-learning interface.
The trunk and synthesis rule are identical to DeepONet, so the proof reduces to showing that the SetONet-Key branch can replicate the input--output map of a DeepONet branch to arbitrary accuracy.
A DeepONet branch computes $\sigma(\sum_j \xi_{ij} g(\boldsymbol{x}_j) + \theta_i)$ in each hidden unit, mixing all sensor values inside the nonlinearity.
In SetONet-Key, the Value Net and readout $\rho_{\mathrm{tok}}$ each act on one input at a time, so this mixing cannot happen inside $\sigma$.
It happens instead in the weighted sum $\sum_j \mathcal{A}_{kj} \boldsymbol{V}_j$, the one operation that touches all sensors at once.
Choosing the mixing weights to match $\xi_{ij}$ makes each token compute $\sum_j \xi_{ij} g(\boldsymbol{x}_j)$, after which the readout applies $\sigma(\cdot + \theta_i)$.

The difficulty is that $\rho_{\mathrm{tok}}$ is the same function for every token, yet each token requires a different bias $\theta_i$.
To distinguish tokens, the Value Net encodes a constant $\lambda$ alongside each sensor reading, producing $\boldsymbol{V}_j = (\lambda g(\boldsymbol{x}_j), \lambda)$.
After a small perturbation of the weights, aggregation deposits a unique numerical tag in each token's second channel.
The readout uses polynomial interpolation on this tag to select the correct bias, recovering the perturbed DeepONet branch output exactly.

Throughout this section, $\|\cdot\|_\infty$ denotes the $\ell^\infty$-norm on $\mathbb R^{d_{\text{out}}}$, and we write $\boldsymbol b_k=\bigl(W\rho_{\mathrm{tok}}(\mathbf P)\bigr)_{k,:}^\top\in\mathbb R^{d_{\text{out}}}$ for the $k$-th branch-coefficient block as in \cref{sec:branch_net}.
For full rigor, we state the result in the scalar-input case $d_u=1$, which is the setting of the classical DeepONet theorem in \cite{lu2021learning}.

\begin{theorem}[Universal approximation of SetONet-Key on the fixed-sensor subclass]
\label{thm:setonet_key_uat}
Assume \(d_u=1\). Let \(\Omega_x\subset\mathbb R^{d_x}\) and \(\Omega_y\subset\mathbb R^{d_y}\) be compact, let \(\mathcal G\subset C(\Omega_x)\) be compact, and let
\[
\mathcal T:\mathcal G\to C(\Omega_y;\mathbb R^{d_{\text{out}}})
\]
be continuous, where \(C(\Omega_x)\) is equipped with the sup norm and \(C(\Omega_y;\mathbb R^{d_{\text{out}}})\) with the sup-\(\ell^\infty\) norm. Consider SetONet-Key as defined in \cref{sec:branch_net,eq:setonet_output_simplified}, and assume the following.
\begin{enumerate}
    \item The positional encoding \(\mathrm{PE}:\Omega_x\to\mathbb R^{d_{\mathrm{PE}}}\) is continuous and injective.

    \item For each fixed sensor count \(M\), the token--sensor mixing rule has the form
    \[
    \mathcal A_{\mathrm{key}}(\mathbf S)=\alpha_M\,a_{\mathrm{mix}}(\mathbf S),
    \]
    where \(\alpha_M>0\) is a deterministic scalar depending only on \(M\), and \(a_{\mathrm{mix}}:\mathbb R\to\mathbb R\) is applied entrywise. Assume that \(a_{\mathrm{mix}}\) is continuous and one-to-one on some interval \(I\subset\mathbb R\), and that
    \[
    0\in \operatorname{int}\bigl(a_{\mathrm{mix}}(I)\bigr).
    \]
    This covers, for example, \(\tanh\).

    \item The Key Net \(k_{\boldsymbol\theta}\), the Value Net \(v_{\boldsymbol\theta}\), the row-wise readout \(\rho_{\mathrm{tok}}\), and the trunk network are MLPs with the same continuous non-polynomial activation \(\sigma\).
\end{enumerate}

Then, for every \(\varepsilon>0\), there exist integers \(M,p,n_{\mathrm{pool}},d_k,d_v\), a fixed sensor set \(\{\boldsymbol x_j\}_{j=1}^{M}\subset\Omega_x\), and SetONet-Key parameters \(\boldsymbol\theta\) such that, when each \(g\in\mathcal G\) is presented to the branch as
\[
\{(\boldsymbol x_j,g(\boldsymbol x_j))\}_{j=1}^{M},
\]
the resulting SetONet-Key operator \(\mathcal T_{\boldsymbol\theta}\) satisfies
\[
\sup_{g\in\mathcal G}\sup_{\boldsymbol y\in\Omega_y}
\bigl\|
\mathcal T(g)(\boldsymbol y)-\mathcal T_{\boldsymbol\theta}(g)(\boldsymbol y)
\bigr\|_\infty
<
\varepsilon.
\]
In particular, the general SetONet-Key family is a universal approximator on the classical fixed-sensor operator-learning subclass.
\end{theorem}

\begin{proof}
Fix \(\varepsilon>0\). For notational convenience, let \(\widehat{\boldsymbol e}_r\) denote the \(r\)-th canonical basis vector in the relevant Euclidean space.

\paragraph{Step 1 (Reduction to stacked DeepONet)}
For each output component \(\ell=1,\dots,d_{\text{out}}\), define
\[
\mathcal T_\ell(g)(\boldsymbol y):=
\bigl(\mathcal T(g)(\boldsymbol y)\bigr)_\ell .
\]
Since the coordinate projection is continuous, each \(\mathcal T_\ell:\mathcal G\to C(\Omega_y)\) is continuous. By the scalar DeepONet theorem of \cite{lu2021learning}, for each \(\ell\) there exist integers \(m_\ell,p_\ell,n_\ell\), a sensor set
\(
X_\ell=\{\boldsymbol x_j^{(\ell)}\}_{j=1}^{m_\ell}\subset\Omega_x,
\)
and coefficients
\[
c_{\ell,i}^r,\;\xi_{\ell,ij}^r,\;\theta_{\ell,i}^r,\;\zeta_{\ell,r}\in\mathbb R,
\qquad
\boldsymbol w_{\ell,r}\in\mathbb R^{d_y},
\]
such that
\[
\mathcal D_\ell(g)(\boldsymbol y)
=
\sum_{r=1}^{p_\ell}
\left(
\sum_{i=1}^{n_\ell}
c_{\ell,i}^r\,
\sigma\!\Bigl(
\sum_{j=1}^{m_\ell}\xi_{\ell,ij}^r\,g(\boldsymbol x_j^{(\ell)})+\theta_{\ell,i}^r
\Bigr)
\right)
\sigma(\boldsymbol w_{\ell,r}\cdot\boldsymbol y+\zeta_{\ell,r})
\]
satisfies
\[
\sup_{g\in\mathcal G}\sup_{\boldsymbol y\in\Omega_y}
\bigl|
\mathcal T_\ell(g)(\boldsymbol y)-\mathcal D_\ell(g)(\boldsymbol y)
\bigr|
<
\frac{\varepsilon}{4}.
\]

Let
\[
X:=\bigcup_{\ell=1}^{d_{\text{out}}}X_\ell=\{\boldsymbol x_j\}_{j=1}^{m},
\qquad
n:=\max_{1\le \ell\le d_{\text{out}}} n_\ell,
\qquad
p:=\sum_{\ell=1}^{d_{\text{out}}}p_\ell .
\]
Extend each scalar approximation to the common sensor set \(X\) by setting the coefficients corresponding to sensors outside \(X_\ell\) equal to zero, and pad each hidden sum with zero coefficients up to width \(n\). Reindex all scalar trunk factors by a single global index \(k=1,\dots,p\), and let \(\ell(k)\in\{1,\dots,d_{\text{out}}\}\) denote the output component associated with index \(k\). Then there exist coefficients
\[
c_i^k,\;\xi_{ij}^k,\;\theta_i^k,\;\zeta_k\in\mathbb R,
\qquad
\boldsymbol w_k\in\mathbb R^{d_y},
\]
such that
\[
\beta_k(g)
:=
\sum_{i=1}^{n}
c_i^k\,
\sigma\!\Bigl(
\sum_{j=1}^{m}\xi_{ij}^k\,g(\boldsymbol x_j)+\theta_i^k
\Bigr),
\qquad
\boldsymbol\tau_k(\boldsymbol y)
:=
\sigma(\boldsymbol w_k\cdot\boldsymbol y+\zeta_k)\,\widehat{\boldsymbol e}_{\ell(k)},
\]
and
\begin{equation}
\sup_{g\in\mathcal G}\sup_{\boldsymbol y\in\Omega_y}
\left\|
\mathcal T(g)(\boldsymbol y)
-
\sum_{k=1}^{p}\beta_k(g)\,\boldsymbol\tau_k(\boldsymbol y)
\right\|_\infty
<
\frac{\varepsilon}{4}.
\label{eq:setonet_key_uat_deeponet}
\end{equation}
The use of \(\|\cdot\|_\infty\) is precisely what makes this stacking step immediate: the vector error is the maximum of the scalar component errors.

\paragraph{Step 2 (Distinct token codes)}
For each pair \((k,i)\), define
\[
\gamma_i^k:=\sum_{j=1}^{m}\xi_{ij}^k.
\]
The equalities \(\gamma_i^k=\gamma_{i'}^{k'}\) define finitely many proper affine hyperplanes in the parameter space of the coefficients \(\{\xi_{ij}^k\}\). Their complement is therefore dense, so we may perturb the coefficients arbitrarily slightly to new coefficients \(\{\widetilde\xi_{ij}^k\}\) such that all numbers
\[
\widetilde\gamma_i^k:=\sum_{j=1}^{m}\widetilde\xi_{ij}^k,
\qquad
1\le k\le p,\;\;1\le i\le n,
\]
are pairwise distinct.

Define
\[
\widetilde\beta_k(g)
:=
\sum_{i=1}^{n}
c_i^k\,
\sigma\!\Bigl(
\sum_{j=1}^{m}\widetilde\xi_{ij}^k\,g(\boldsymbol x_j)+\theta_i^k
\Bigr),
\]
and
\[
\widetilde{\mathcal D}(g)(\boldsymbol y)
:=
\sum_{k=1}^{p}\widetilde\beta_k(g)\,\boldsymbol\tau_k(\boldsymbol y).
\]
Because \(\widetilde{\mathcal D}(g)(\boldsymbol y)\) depends continuously on the finitely many parameters \(\{\widetilde\xi_{ij}^k\}\), and because \(\mathcal G\times\Omega_y\) is compact, the perturbation can be chosen small enough that
\begin{equation}
\sup_{g\in\mathcal G}\sup_{\boldsymbol y\in\Omega_y}
\bigl\|
\mathcal T(g)(\boldsymbol y)-\widetilde{\mathcal D}(g)(\boldsymbol y)
\bigr\|_\infty
<
\frac{\varepsilon}{2}.
\label{eq:setonet_key_uat_perturbed}
\end{equation}

\paragraph{Step 3 (Exact SetONet-Key realization of the perturbed network)}
The idea is to make the Key Net isolate each sensor via interpolation, let the Value Net carry both the sensor reading and a constant in a two-dimensional output, and use the constant channel after aggregation as a per-token identifier so that the shared readout can assign the correct bias and output component to each token.

Set
\[
M=m,\qquad
d_k=m,\qquad
d_v=2,\qquad
n_{\mathrm{pool}}=pn,
\]
and index the \(n_{\mathrm{pool}}\) tokens by pairs \((k,i)\), where \(1\le k\le p\) and \(1\le i\le n\).

Since \(\mathrm{PE}\) is injective, the encoded sensor locations
\[
\boldsymbol e_j:=\mathrm{PE}(\boldsymbol x_j)\in\mathbb R^{d_{\mathrm{PE}}},
\qquad
j=1,\dots,m,
\]
are pairwise distinct. Define, for \(j=1,\dots,m\),
\[
\eta_j(\boldsymbol z)
:=
\prod_{\substack{r=1\\ r\neq j}}^{m}
\frac{\|\boldsymbol z-\boldsymbol e_r\|_2^2}{\|\boldsymbol e_j-\boldsymbol e_r\|_2^2},
\qquad
k^\star(\boldsymbol z)
:=
\bigl(\eta_1(\boldsymbol z),\dots,\eta_m(\boldsymbol z)\bigr).
\]
Then \(k^\star:\mathbb R^{d_{\mathrm{PE}}}\to\mathbb R^m\) is continuous and satisfies
\[
k^\star(\boldsymbol e_j)=\widehat{\boldsymbol e}_j\in\mathbb R^m,
\qquad
j=1,\dots,m.
\]

Let \(\alpha_m>0\) be the deterministic normalization factor in \(\mathcal A_{\mathrm{key}}\) when \(M=m\). Since \(0\in \operatorname{int}(a_{\mathrm{mix}}(I))\) and only finitely many coefficients \(\widetilde\xi_{ij}^k\) occur, there exists \(\lambda>0\) such that
\[
\frac{\widetilde\xi_{ij}^k}{\alpha_m\lambda}\in a_{\mathrm{mix}}(I)
\qquad
\text{for all }i,j,k.
\]
Define the ideal Value Net
\[
v^\star(u):=(\lambda u,\lambda)\in\mathbb R^2.
\]
For the token indexed by \((k,i)\), define the query vector
\[
\boldsymbol q_{k,i}
=
\sqrt{d_k}\,
\Bigl(
a_{\mathrm{mix}}^{-1}\!\bigl(\widetilde\xi_{i1}^k/(\alpha_m\lambda)\bigr),
\dots,
a_{\mathrm{mix}}^{-1}\!\bigl(\widetilde\xi_{im}^k/(\alpha_m\lambda)\bigr)
\Bigr)
\in\mathbb R^{m}.
\]

Now present \(g\in\mathcal G\) to the branch as
\(
\{(\boldsymbol x_j,g(\boldsymbol x_j))\}_{j=1}^{m}.
\)
With the ideal Key Net \(k^\star\) and ideal Value Net \(v^\star\),
\[
\boldsymbol K_j
=
k^\star(\boldsymbol e_j)
=
\widehat{\boldsymbol e}_j,
\qquad
\boldsymbol V_j
=
v^\star(g(\boldsymbol x_j))
=
(\lambda g(\boldsymbol x_j),\lambda).
\]
Hence, for the token \((k,i)\), the score against sensor \(j\) is
\[
\frac{\boldsymbol q_{k,i}\cdot \boldsymbol K_j}{\sqrt{d_k}}
=
a_{\mathrm{mix}}^{-1}\!\bigl(\widetilde\xi_{ij}^k/(\alpha_m\lambda)\bigr).
\]
Therefore,
\[
\mathcal A_{\mathrm{key}}\!\left(
\frac{\boldsymbol q_{k,i}\mathbf K^\top}{\sqrt{d_k}}
\right)_j
=
\alpha_m\,
a_{\mathrm{mix}}\!\left(
a_{\mathrm{mix}}^{-1}\!\bigl(\widetilde\xi_{ij}^k/(\alpha_m\lambda)\bigr)
\right)
=
\frac{\widetilde\xi_{ij}^k}{\lambda}.
\]
Thus the \((k,i)\)-th row of the ideal token matrix
\[
\mathbf P^\star
=
\operatorname{TSA}_{\mathcal A_{\mathrm{key}}}(\mathbf Q,\mathbf K,\mathbf V)
\in \mathbb R^{pn\times 2}
\]
is exactly
\begin{align}
\mathbf P^\star_{(k,i),:}
&=
\sum_{j=1}^{m}
\frac{\widetilde\xi_{ij}^k}{\lambda}\,
(\lambda g(\boldsymbol x_j),\lambda)
\nonumber\\
&=
\left(
\sum_{j=1}^{m}\widetilde\xi_{ij}^k\,g(\boldsymbol x_j),
\sum_{j=1}^{m}\widetilde\xi_{ij}^k
\right)
=
\left(
\sum_{j=1}^{m}\widetilde\xi_{ij}^k\,g(\boldsymbol x_j),
\widetilde\gamma_i^k
\right).
\label{eq:setonet_key_uat_token_row}
\end{align}

Because the codes \(\{\widetilde\gamma_i^k\}\) are pairwise distinct, the Lagrange polynomials
\[
L_{k,i}(c)
:=
\prod_{(k',i')\neq(k,i)}
\frac{c-\widetilde\gamma_{i'}^{k'}}{\widetilde\gamma_i^k-\widetilde\gamma_{i'}^{k'}},
\qquad
1\le k\le p,\;\;1\le i\le n,
\]
are well-defined. Define the ideal row-wise readout
\[
\rho_{\mathrm{tok}}^\star(s,c)
:=
\sum_{k=1}^{p}\sum_{i=1}^{n}
L_{k,i}(c)\,
\sigma(s+\theta_i^k)\,
\widehat{\boldsymbol e}_{\ell(k)},
\qquad
(s,c)\in\mathbb R^2.
\]
This function is continuous. By the interpolation property
\(
L_{k,i}(\widetilde\gamma_{i'}^{k'})=\delta_{kk'}\delta_{ii'},
\)
\cref{eq:setonet_key_uat_token_row} yields
\[
\rho_{\mathrm{tok}}^\star\!\bigl(\mathbf P^\star_{(k,i),:}\bigr)
=
\sigma\!\Bigl(
\sum_{j=1}^{m}\widetilde\xi_{ij}^k\,g(\boldsymbol x_j)+\theta_i^k
\Bigr)\,
\widehat{\boldsymbol e}_{\ell(k)}.
\]

Choose \(W\in\mathbb R^{p\times (pn)}\) so that
\[
W_{k,(k',i)}
:=
c_i^k\,\delta_{kk'},
\qquad
1\le k,k'\le p,\;\;1\le i\le n.
\]
Then
\[
W\rho_{\mathrm{tok}}^\star(\mathbf P^\star)\in\mathbb R^{p\times d_{\text{out}}},
\]
and, for each \(k=1,\dots,p\),
\[
\bigl(W\rho_{\mathrm{tok}}^\star(\mathbf P^\star)\bigr)_{k,:}
=
\sum_{i=1}^{n}
c_i^k\,
\rho_{\mathrm{tok}}^\star\!\bigl(\mathbf P^\star_{(k,i),:}\bigr)^\top
=
\widetilde\beta_k(g)\,\widehat{\boldsymbol e}_{\ell(k)}^\top.
\]
Hence
\[
\boldsymbol b_k^\star(g)
:=
\bigl(W\rho_{\mathrm{tok}}^\star(\mathbf P^\star)\bigr)_{k,:}^\top
=
\widetilde\beta_k(g)\,\widehat{\boldsymbol e}_{\ell(k)}
\in\mathbb R^{d_{\text{out}}}.
\]

Set
\[
\boldsymbol t_k^\star(\boldsymbol y)
:=
\boldsymbol\tau_k(\boldsymbol y)
=
\sigma(\boldsymbol w_k\cdot\boldsymbol y+\zeta_k)\,\widehat{\boldsymbol e}_{\ell(k)},
\qquad
\boldsymbol b_0=\mathbf 0.
\]
Then
\[
\boldsymbol b_k^\star(g)\odot \boldsymbol t_k^\star(\boldsymbol y)
=
\widetilde\beta_k(g)\,\boldsymbol\tau_k(\boldsymbol y),
\]
so the SetONet-Key synthesis rule in \cref{eq:setonet_output_simplified} reproduces \(\widetilde{\mathcal D}(g)(\boldsymbol y)\) exactly.

\paragraph{Step 4 (Replacement by admissible MLP subnets)}
The ideal maps \(k^\star\), \(v^\star\), and \(\rho_{\mathrm{tok}}^\star\) are continuous on the relevant compact sets. Indeed,
\[
K_{\mathrm{PE}}:=\mathrm{PE}(\Omega_x)\subset\mathbb R^{d_{\mathrm{PE}}}
\]
is compact because \(\mathrm{PE}\) is continuous and \(\Omega_x\) is compact, while
\[
K_u
:=
\{g(\boldsymbol x_j): g\in\mathcal G,\; j=1,\dots,m\}\subset\mathbb R
\]
is compact because the evaluation maps \(g\mapsto g(\boldsymbol x_j)\) are continuous and \(\mathcal G\) is compact. Likewise,
\[
K_P
:=
\bigcup_{k=1}^{p}\bigcup_{i=1}^{n}
\left\{
\left(
\sum_{j=1}^{m}\widetilde\xi_{ij}^k\,g(\boldsymbol x_j),
\widetilde\gamma_i^k
\right)
:\;
g\in\mathcal G
\right\}
\subset\mathbb R^2
\]
is compact.

By the finite-dimensional universal approximation theorem for MLPs with continuous non-polynomial activation \cite{leshno1993multilayer}, each scalar coordinate of \(k^\star\), \(v^\star\), and \(\rho_{\mathrm{tok}}^\star\) can be approximated uniformly on \(K_{\mathrm{PE}}\), \(K_u\), and on a compact neighborhood of \(K_P\), respectively. Stacking these coordinatewise approximants in parallel yields actual MLPs \(k_{\boldsymbol\theta}\), \(v_{\boldsymbol\theta}\), and \(\rho_{\mathrm{tok}}\) that approximate the ideal maps uniformly as closely as desired. The query tokens \(\mathbf Q\) and the linear map \(W\) are free parameters, so they are realized exactly.

The trunk basis \(\{\boldsymbol t_k^\star\}_{k=1}^{p}\) is also realized exactly by an admissible trunk MLP: the \(k\)-th hidden unit computes \(\sigma(\boldsymbol w_k\cdot\boldsymbol y+\zeta_k)\), and the output layer routes this scalar to the \(\ell(k)\)-th component of the \(k\)-th basis block.

Because the sensor set is fixed and finite, and because matrix multiplication, \(a_{\mathrm{mix}}\), the row-wise readout, reshaping, and the synthesis map in \cref{eq:setonet_output_simplified} are continuous on the relevant compact sets, the full SetONet-Key operator depends continuously on these subnet approximations. Therefore the above MLP approximants can be chosen so that the resulting SetONet-Key operator \(\mathcal T_{\boldsymbol\theta}\) satisfies
\[
\sup_{g\in\mathcal G}\sup_{\boldsymbol y\in\Omega_y}
\bigl\|
\widetilde{\mathcal D}(g)(\boldsymbol y)-\mathcal T_{\boldsymbol\theta}(g)(\boldsymbol y)
\bigr\|_\infty
<
\frac{\varepsilon}{2}.
\]
Combining this estimate with \cref{eq:setonet_key_uat_perturbed} and using the triangle inequality gives
\[
\sup_{g\in\mathcal G}\sup_{\boldsymbol y\in\Omega_y}
\bigl\|
\mathcal T(g)(\boldsymbol y)-\mathcal T_{\boldsymbol\theta}(g)(\boldsymbol y)
\bigr\|_\infty
<
\varepsilon.
\]
This proves the theorem.
\end{proof}

\paragraph{Remarks}

First, the construction relies essentially on the \emph{non-normalized} mixing rule in assumption (2). It does not apply to row-wise softmax normalization.

Second, the \(\ell^\infty\)-norm is adopted because it makes the vector-valued stacking in Step 1 immediate: the supremum over components of the individual scalar errors transfers directly to the vector error without dimensional constants. Since all norms on \(\mathbb R^{d_{\text{out}}}\) are equivalent, this choice entails no loss of generality.

Third, the theorem is stated for fixed sensor locations. 
The variable-cardinality and point-cloud settings studied in \cref{sec:results} fall outside this scope, and a universal approximation result for those regimes remains open. 
The experimental results in \cref{sec:results} provide empirical evidence that the architecture generalizes beyond the fixed-sensor subclass.

\section{Implementation Details}
\label{sec:implementation}

All models are implemented in PyTorch~\cite{paszke2019pytorch} within a unified training and evaluation pipeline. For fair comparison, we use identical dataset splits, sensor protocols, and optimization settings across methods within each benchmark family, while allowing benchmark-dependent input/output dimensions and latent dimension $p$. We compare SetONet variants against standard DeepONet and VIDON~\cite{prasthofer2022variable}.

\paragraph{Architecture}
SetONet follows the branch--trunk construction described in \cref{sec:setonet_architecture}. For SetONet (Attention/Mean/Sum), sensor coordinates are encoded with sinusoidal positional features ($d_{\mathrm{PE}}=64$), followed by a shared value network and permutation-invariant aggregation; attention pooling uses one pooling token with four heads to balance between performance and computation efficiency. SetONet-Key uses the dedicated key and value networks described in \cref{sec:branch_net}. We fix $p=32$ for all 1D benchmark problems and $p=128$ for 2D benchmark problems, all trunk networks are defined as ReLU networks~\cite{nair2010rectified} of width 256. This yields a controlled comparison at matched latent and trunk scale while preserving each model's native branch design. DeepONet uses a standard MLP branch, whereas VIDON follows the reference multi-head attention architecture proposed in \cite{prasthofer2022variable}, employing four attention heads with an encoder dimension of $40$ and a head output size of 64. For benchmarks where DeepONet is applicable, SetONet variants are configured with parameter counts matched to or below DeepONet. VIDON retains its original architectural scaling from~\cite{prasthofer2022variable} and is therefore larger (${\sim}700\text{k}$ trainable parameters) than DeepONet and SetONet variants (${\sim}300\text{k}$ trainable parameters). Exact parameter counts for different examples are reported in \cref{tab:param_counts}.

\paragraph{Training}
All models are trained with Adam optimizer~\cite{kingma2014adam}, using initial learning rate at $5\times10^{-4}$ with no weight decay and standard gradient clipping to ensure training stability. We use a total of 125{,}000 optimization steps (batch size 64) for the \emph{Derivative}, \emph{Integral}, \emph{Darcy 1D}, and \emph{Elastic Plate} benchmarks, and 50{,}000 steps (batch size 32) for \emph{Heat}, \emph{Advection--Diffusion}, \emph{Diffraction}, and \emph{Optimal Transport} examples. Learning-rate decay follows a multi-step schedule. For the 125k-step training regime, decay milestones are set at 25k and 75k
steps, with multiplicative factors of 0.2 and 0.5, respectively. For the 50k-step regime, milestones occur at 15k and 30k steps, with factors of $0.2$ and $0.5$. We conduct five random seeded trials for all of our examples to ensure our claims are valid and robust to test, we report aggregated metrics such as mean and standard deviation for all examples. 

Complete list of hyperparameters, architectural configurations and dataset generation settings are provided in \cref{app:config,app:datasets}.

\section{Numerical Experiments}
\label{sec:results}

We evaluate SetONet on two groups of benchmarks—against DeepONet and VIDON baselines for the first group, and comparison against VIDON only for the second, since those tasks use unordered point source inputs for which standard DeepONet is not directly applicable, demonstrating the broader applicability of our approach. 

Our experimental setup and results emphasize two claims. First, on \emph{Derivative}, \emph{Anti-derivative}, \emph{Darcy 1D}, and \emph{Elastic Plate} (2D), we perform parameter-matched comparisons with approximately 250,000 trainable parameters (see~\cref{tab:param_counts}) for both SetONet and DeepONet. We evaluate SetONet under three sensor configurations (Fixed, Variable, Drop-off) and compare against DeepONet and VIDON on these configurations. SetONet achieves lower error than both DeepONet and VIDON on fixed layouts and sustains high accuracy under changing layouts and evaluation-time drop-off, reflecting robustness that follows directly from its permutation-invariant set encoder. 

Second, on \emph{Heat} (2D), \emph{Advection–Diffusion} (2D), \emph{Phase-Screen Diffraction} (2D), and \emph{Optimal Transport} (2D), inputs arrive as point clouds data. SetONet solves these end-to-end with a single-stage training recipe and a lightweight architecture—without rasterization, ad-hoc imputation, or multi-stage pipelines. Importantly, SetONet achieves superior performance across all examples compared to VIDON, while using smaller model sizes and requiring less training time.

Taken together, the results position SetONet as a drop-in replacement for DeepONet on classical benchmarks and a practical, more general framework for operator learning with variable or unstructured inputs.

\subsection{Sensor Configurations for Testing}
For the first set of examples, we evaluate all methods under three sensor configurations. These configurations reflect common scenarios, including sensor location variability and sensor failure. 

\paragraph{Fixed (static layout)}
A single set of \(M\) fixed sensor locations is chosen once and used for all samples, batches, and epochs during both training and evaluation. No sensors are removed. This setting is used for all benchmarks (\emph{Derivative}, \emph{Anti-derivative}, \emph{Darcy 1D}, \emph{Elastic 2D}).

\paragraph{Variable (changing sensor layouts during training \emph{and} evaluation)}
A given model is both trained and tested under changing discretizations of the input functions. Concretely, for the \emph{Derivative} and \emph{Anti-derivative} examples we uniformly resample the \(M\) sensor locations at the start of each batch (all samples in the batch share that layout; \(M\) is fixed). For \emph{Darcy 1D} and \emph{Elastic 2D} we apply sensor dropout with nearest-neighbor replacement during both training and evaluation, preserving input size and ensuring a fair comparison across methods. 

\paragraph{Drop-off (evaluation-only ablation with replacement)}
Models are evaluated under a controlled sensor-failure scenario: at test time, we uniformly drop \(20\%\) of sensors and fill each dropped index by duplicating the nearest retained sensor’s location–value pair, so the input size remains \(M\). For the \emph{Derivative} and \emph{Anti-derivative} tasks, training follows the \emph{Variable} configuration, and the \(20\%\) drop-off is applied on top of the resampled layout only at evaluation. For \emph{Darcy 1D} and \emph{Elastic 2D}, training follows the \emph{Fixed} configuration, and the same \(20\%\) drop-off is applied only at evaluation.

All subsequent subsections report results under these \emph{Fixed}, \emph{Variable}, and \emph{Drop-off} configurations as described above.


\begin{figure}[ht]
    \centering
    \includegraphics[width=\textwidth]{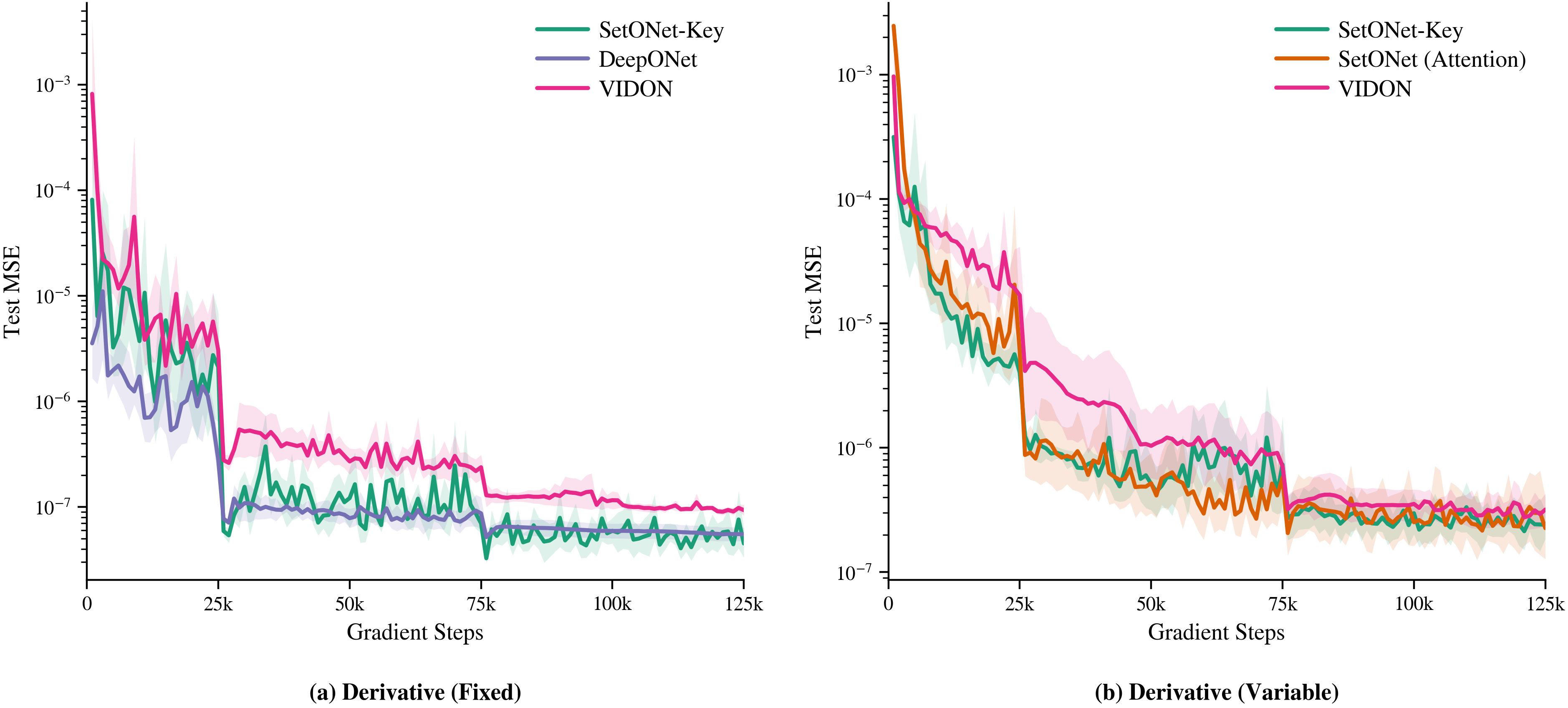}
    \caption{Test MSE histories for the Derivative benchmark under (a) the \emph{fixed} sensor configuration and (b) the \emph{variable} sensor configuration. Solid curves show the mean over five independent random initializations, and shaded bands indicate one standard deviation across seeds in $\log_{10}$-space. Curves are reported without temporal smoothing.}
    \label{fig:derivative_loss_main}
\end{figure}

\begin{figure}[ht]
    \centering
    \includegraphics[width=\textwidth]{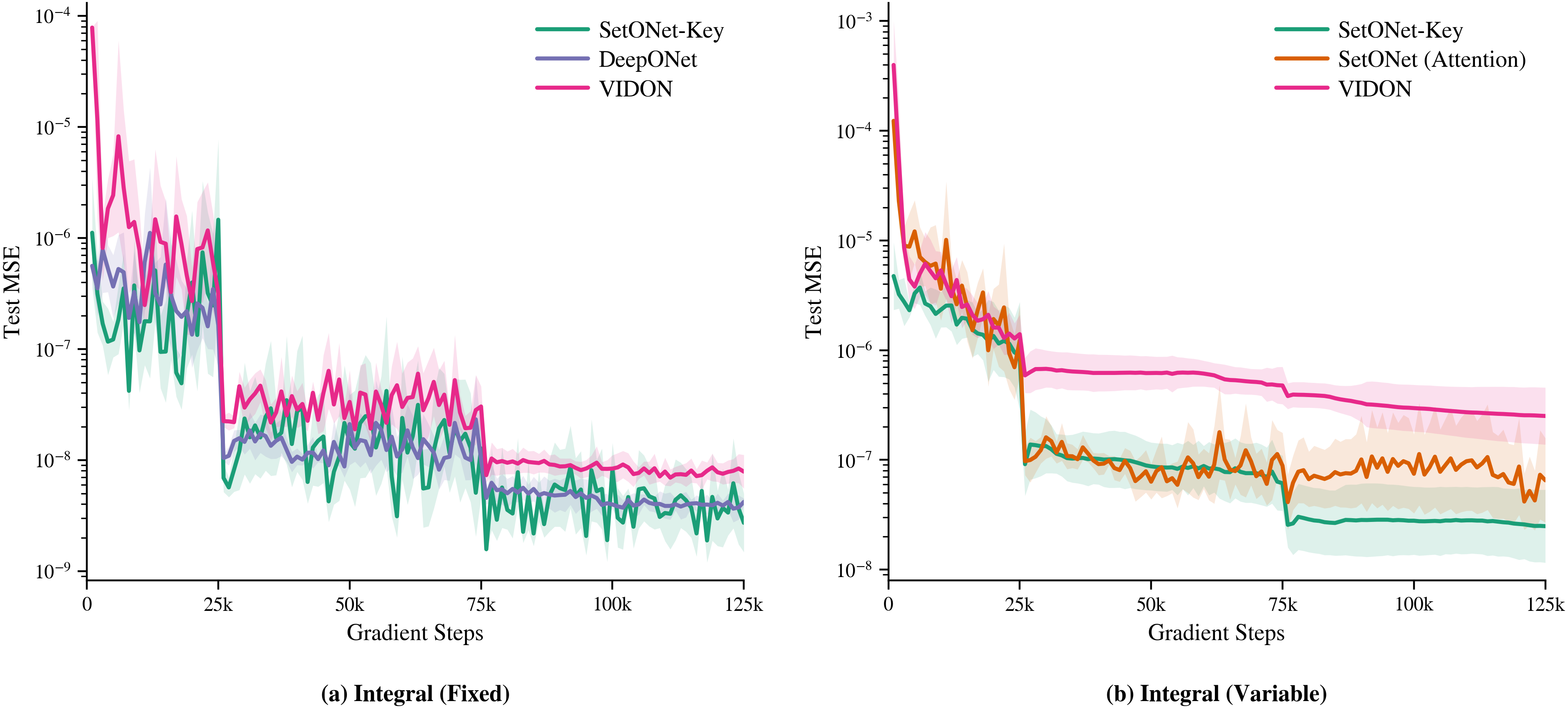}
    \caption{Test MSE histories for the Integral benchmark under (a) the \emph{fixed} sensor configuration and (b) the \emph{variable} sensor configuration. Solid curves show the mean over five independent random initializations, and shaded bands indicate one standard deviation across seeds in $\log_{10}$-space. Curves are reported without temporal smoothing.}
    \label{fig:integral_loss_main}
\end{figure}

We first compare SetONet, DeepONet and VIDON on two 1D benchmark tasks: differentiation and integration. These provide simple but controlled tests of operator learning under different input sampling conditions.

In both cases, the underlying function family is
\[
f(x)=ax^3+bx^2+cx+e\sin x,
\qquad
x\in[-1,1],
\]
where \(a,b,c,e\) are randomly sampled coefficients, with zero integration constant. For differentiation, the operator is
\[
\mathcal{T}: f(x) \mapsto \frac{d}{dx} f(x).
\]
For the integration task, we learn the inverse mapping
\[
\mathcal{T}: f'(x) \mapsto f(x)=\int_0^x f'(s)\, ds,
\]
which is well defined because $f(0)=0$. The functions are resampled independently during training.

We present quantitative results for both examples under different sampling scenarios in~\cref{tab:setonet_results_pow}. When measured in $\ell_2$ relative error, SetONet outperforms both standard DeepONet and VIDON. More importantly, its accuracy remains consistent across different sensor configurations, with only limited loss in performance. This highlights the efficacy and robustness of our proposed approach. In addition, validation curves and qualitative results are shown in~\cref{fig:derivative_loss_main}, \cref{fig:integral_loss_main}, \cref{fig:combined_derivative_prediction}, and \cref{fig:Integral_combined_prediction}, further supporting our claims.

\begin{figure}[ht]
    \centering
    \includegraphics[width=\textwidth]{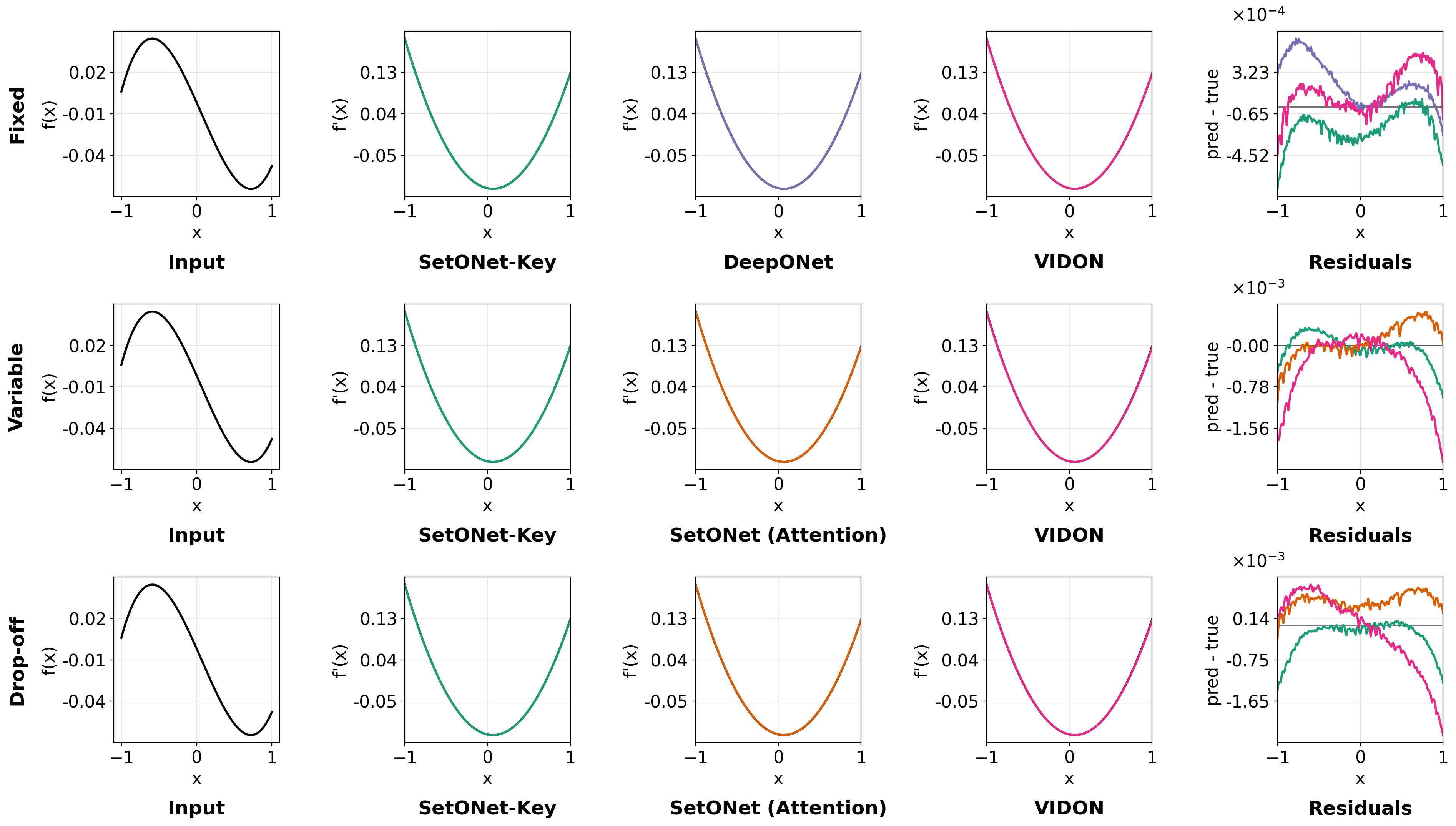}
    \caption{Qualitative comparison for the 1D derivative benchmark under three sensor regimes. Rows correspond to fixed, variable, and drop-off settings; columns show the input \(f(x)\) (black), three model predictions against the ground truth (black), and residuals \((\mathrm{pred}-\mathrm{true})\). In the fixed setting, the compared models are SetONet-Key (green), DeepONet (purple), and VIDON (magenta); in the variable and drop-off settings, the compared models are SetONet-Key (green), SetONet (Attention) (orange), and VIDON (magenta).}
    \label{fig:combined_derivative_prediction}
\end{figure}

\begin{figure}[ht]
    \centering
    \includegraphics[width=\textwidth]{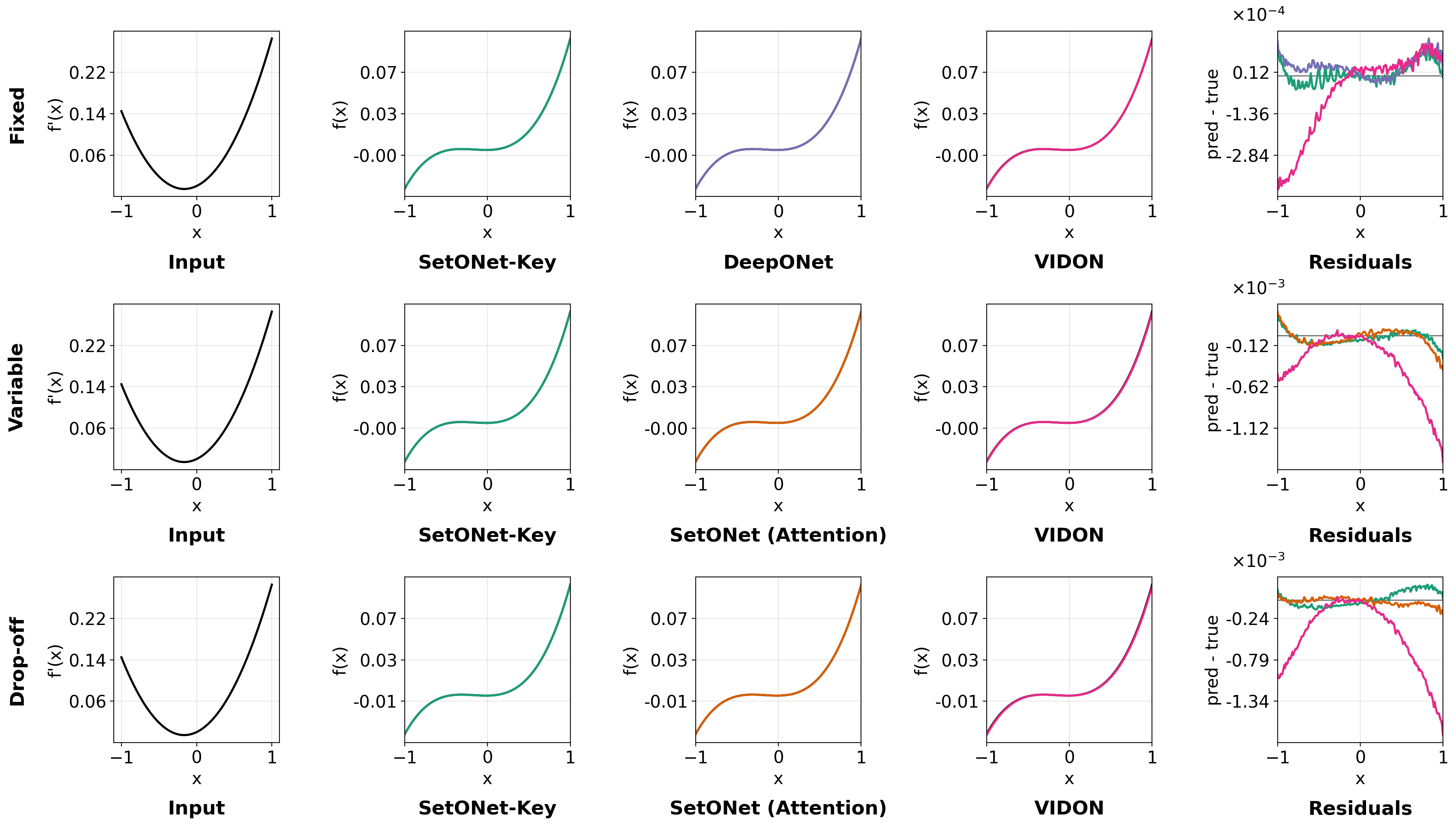}
    \caption{Qualitative comparison for the 1D integral benchmark under three sensor regimes. Rows correspond to fixed, variable, and drop-off settings; columns show the input \(f'(x)\) (black), three model predictions against the ground truth \(f(x)\) (black), and residuals \((\mathrm{pred}-\mathrm{true})\). In the fixed setting, the compared models are SetONet-Key (green), DeepONet (purple), and VIDON (magenta); in the variable and drop-off settings, the compared models are SetONet-Key (green), SetONet (Attention) (orange), and VIDON (magenta).}
    \label{fig:Integral_combined_prediction}
\end{figure}

\begin{table}[htbp]
\centering
\caption{Test relative $\ell_2$ error (mean $\pm$ std) comparing SetONet variants, DeepONet, and VIDON on benchmark problems across three sensor configurations (Fixed, Variable, Drop-off). `---' denotes not applicable.}
\label{tab:setonet_results_pow}
\setlength{\tabcolsep}{6pt}
\renewcommand{\arraystretch}{1.12}
\resizebox{\linewidth}{!}{%
\begin{tabular}{@{}lccc@{}}
\toprule
\multirow{2}{*}{\textbf{Model}} & \multicolumn{3}{c}{\textbf{Sensor Configuration}} \\
\cmidrule(lr){2-4}
& \textbf{Fixed} & \textbf{Variable} & \textbf{Drop-off} \\
\midrule
\multicolumn{4}{c}{\textbf{Derivative}} \\
\midrule
SetONet-Key & \cellcolor{gray!20}\sci{1.97}{-3} $\pm$ \sci{3.20}{-4} & \sci{4.24}{-3} $\pm$ \sci{5.38}{-4} & \sci{5.64}{-3} $\pm$ \sci{1.88}{-3} \\
SetONet (Attention) & \sci{3.09}{-3} $\pm$ \sci{6.76}{-4} & \cellcolor{gray!20}\sci{4.11}{-3} $\pm$ \sci{1.28}{-3} & \cellcolor{gray!20}\sci{4.23}{-3} $\pm$ \sci{1.27}{-3} \\
SetONet (Mean) & \sci{2.80}{-3} $\pm$ \sci{3.00}{-4} & \sci{1.78}{-2} $\pm$ \sci{1.24}{-3} & \sci{1.78}{-2} $\pm$ \sci{1.19}{-3} \\
SetONet (Sum) & \sci{3.12}{-3} $\pm$ \sci{9.71}{-4} & \sci{1.65}{-2} $\pm$ \sci{1.11}{-3} & \sci{1.66}{-2} $\pm$ \sci{1.03}{-3} \\
DeepONet & \sci{2.24}{-3} $\pm$ \sci{1.88}{-4} & \textemdash & \textemdash \\
VIDON & \sci{2.91}{-3} $\pm$ \sci{1.83}{-4} & \sci{4.86}{-3} $\pm$ \sci{7.31}{-4} & \sci{9.83}{-3} $\pm$ \sci{6.01}{-4} \\
\midrule
\multicolumn{4}{c}{\textbf{Integral}} \\
\midrule
SetONet-Key & \cellcolor{gray!20}\sci{1.11}{-3} $\pm$ \sci{3.12}{-4} & \cellcolor{gray!20}\sci{2.83}{-3} $\pm$ \sci{1.11}{-3} & \cellcolor{gray!20}\sci{3.36}{-3} $\pm$ \sci{1.18}{-3} \\
SetONet (Attention) & \sci{2.08}{-3} $\pm$ \sci{1.46}{-4} & \sci{5.34}{-3} $\pm$ \sci{3.55}{-3} & \sci{5.45}{-3} $\pm$ \sci{3.51}{-3} \\
SetONet (Mean) & \sci{1.78}{-3} $\pm$ \sci{9.21}{-4} & \sci{9.69}{-3} $\pm$ \sci{6.42}{-4} & \sci{9.80}{-3} $\pm$ \sci{6.75}{-4} \\
SetONet (Sum) & \sci{1.97}{-3} $\pm$ \sci{3.62}{-4} & \sci{9.72}{-3} $\pm$ \sci{3.81}{-4} & \sci{9.79}{-3} $\pm$ \sci{4.13}{-4} \\
DeepONet & \sci{1.59}{-3} $\pm$ \sci{5.40}{-5} & \textemdash & \textemdash \\
VIDON & \sci{2.07}{-3} $\pm$ \sci{3.21}{-4} & \sci{9.71}{-3} $\pm$ \sci{3.20}{-3} & \sci{9.96}{-3} $\pm$ \sci{3.14}{-3} \\
\midrule
\multicolumn{4}{c}{\textbf{Darcy 1D}} \\
\midrule
SetONet-Key & \cellcolor{gray!20}\sci{2.97}{-3} $\pm$ \sci{7.58}{-4} & \cellcolor{gray!20}\sci{3.58}{-3} $\pm$ \sci{3.58}{-4} & \cellcolor{gray!20}\sci{3.77}{-3} $\pm$ \sci{5.60}{-4} \\
SetONet (Attention) & \sci{6.89}{-3} $\pm$ \sci{8.03}{-4} & \sci{8.21}{-3} $\pm$ \sci{7.09}{-4} & \sci{7.72}{-3} $\pm$ \sci{7.46}{-4} \\
SetONet (Mean) & \sci{5.37}{-3} $\pm$ \sci{1.85}{-4} & \sci{7.08}{-3} $\pm$ \sci{5.94}{-4} & \sci{6.29}{-3} $\pm$ \sci{1.59}{-4} \\
SetONet (Sum) & \sci{5.53}{-3} $\pm$ \sci{5.32}{-4} & \sci{7.15}{-3} $\pm$ \sci{9.61}{-4} & \sci{6.44}{-3} $\pm$ \sci{3.95}{-4} \\
DeepONet & \sci{4.40}{-3} $\pm$ \sci{3.17}{-4} & \textemdash & \textemdash \\
VIDON & \sci{9.91}{-3} $\pm$ \sci{1.12}{-3} & \sci{9.48}{-3} $\pm$ \sci{1.37}{-3} & \sci{1.05}{-2} $\pm$ \sci{1.09}{-3} \\
\midrule
\multicolumn{4}{c}{\textbf{Elastic Plate}} \\
\midrule
SetONet-Key & \cellcolor{gray!20}\sci{1.35}{-3} $\pm$ \sci{1.99}{-4} & \cellcolor{gray!20}\sci{3.00}{-3} $\pm$ \sci{1.75}{-4} & \cellcolor{gray!20}\sci{2.74}{-3} $\pm$ \sci{1.81}{-4} \\
SetONet (Attention) & \sci{3.20}{-3} $\pm$ \sci{4.45}{-4} & \sci{4.36}{-3} $\pm$ \sci{1.93}{-4} & \sci{4.09}{-3} $\pm$ \sci{2.80}{-4} \\
SetONet (Mean) & \sci{2.57}{-3} $\pm$ \sci{5.76}{-4} & \sci{3.95}{-3} $\pm$ \sci{4.43}{-4} & \sci{3.77}{-3} $\pm$ \sci{5.65}{-4} \\
SetONet (Sum) & \sci{2.64}{-3} $\pm$ \sci{3.66}{-4} & \sci{4.12}{-3} $\pm$ \sci{2.92}{-4} & \sci{3.69}{-3} $\pm$ \sci{2.37}{-4} \\
DeepONet & \sci{5.60}{-3} $\pm$ \sci{9.98}{-4} & \textemdash & \textemdash \\
VIDON & \sci{5.24}{-3} $\pm$ \sci{1.67}{-4} & \sci{6.38}{-3} $\pm$ \sci{7.05}{-4} & \sci{6.09}{-3} $\pm$ \sci{2.82}{-4} \\
\bottomrule
\end{tabular}%
}
\end{table}

\begin{table}[htbp]
\centering
\caption{Training time per epoch (in seconds) comparing SetONet variants, DeepONet, and VIDON on the Elastic Plate example.}
\label{tab:training_time}
\fontsize{10pt}{14.4pt}\selectfont
\setlength{\tabcolsep}{5pt}
\renewcommand{\arraystretch}{1.12}
\begin{tabular}{@{}ccccccc@{}}
\toprule
\multirow{2}{*}{\textbf{SetONet-Key}}
& \multicolumn{3}{c}{\textbf{SetONet}} 
& \multirow{2}{*}{\textbf{DeepONet}} 
& \multirow{2}{*}{\textbf{VIDON}} \\
\cmidrule(lr){2-4}
& \textbf{Attention} & \textbf{Mean} & \textbf{Sum} &  &  \\
\midrule
3.744 $\pm$ 0.004 & 3.623 $\pm$ 0.004 & 3.323 $\pm$ 0.003 & 3.316 $\pm$ 0.001 & 3.580 $\pm$ 0.010 & 5.628 $\pm$ 0.007 \\
\bottomrule
\end{tabular}
\end{table}

\subsection{1D Darcy Flow}
\label{subsec:results-1ddarcy}
In this example, we aim to learn the nonlinear operator for a 1D Darcy system. The problem of interest can be written as
\begin{equation}
    \frac{d}{dx}
    \left(
    -\kappa(u(x)) \frac{d u}{dx}
    \right)
     = f(x), \quad x\in[0, 1], 
\end{equation}
where the solution-dependent permeability is $\kappa(u(x)) = 0.2 + u^2(x)$ and the forcing term is a Gaussian random field $f(x) \sim \mathrm{GP}(0, k(x, x'))$ with
\[
k(x, x') = \sigma^2 \exp\!\left(- \frac{(x - x')^2}{2\ell_x^2}\right), \qquad \ell_x = 0.04, \quad \sigma^2 = 1.0.
\]
Homogeneous Dirichlet boundary conditions $u(0) = u(1) = 0$ are imposed at the boundaries. 
The problem is obtained from \cite{ingebrand2025basis} and we use a finite-difference solver to generate the necessary data. The Darcy flow example serves as a common benchmark for nonlinear operator learning tasks.

\begin{figure}[ht]
    \centering
    \includegraphics[width=\textwidth]{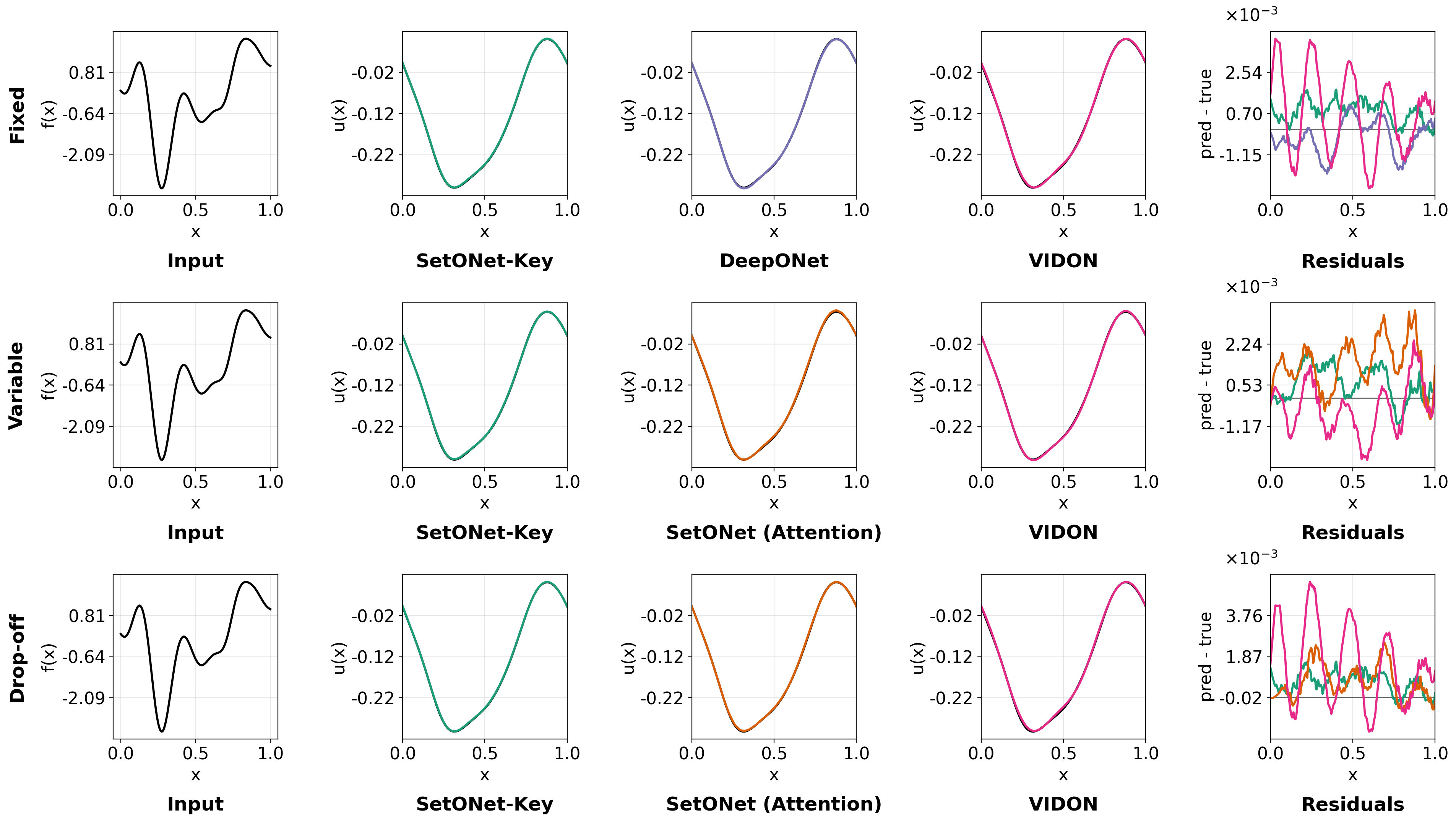}
    \caption{Qualitative comparison for the 1D Darcy flow benchmark under three sensor regimes. Rows correspond to fixed, variable, and drop-off settings; columns show the input \(f(x)\) (black), three model predictions against the ground truth \(u(x)\) (black), and residuals \((\mathrm{pred}-\mathrm{true})\). In the fixed setting, the compared models are SetONet-Key (green), DeepONet (purple), and VIDON (magenta); in the variable and drop-off settings, the compared models are SetONet-Key (green), SetONet (Attention) (orange), and VIDON (magenta).}
    \label{fig:Darcy_combined_prediction}
\end{figure}

We display the learning results in~\cref{fig:darcy_elastic_loss} across different trials and the relative error in~\cref{tab:setonet_results_pow}. In general, SetONet achieves comparable performance to the baselines across all tested scenarios. Furthermore, we highlight that, despite the problem being more difficult under the \emph{variable} and \emph{drop-off} settings, we do not observe a significant decrease in SetONet performance, demonstrating its strength and robustness. Additionally, we note that training with SetONet is stable and reliable, as shown in~\cref{fig:darcy_elastic_loss}; in fact, we observe similar variance across seeded trials compared to baseline methods. Qualitatively, SetONet achieves highly accurate predictions in all tested sensor layouts (see \cref{fig:Darcy_combined_prediction}).

\subsection{Elastic Plate}
\label{subsec:results-elastic}

\begin{figure}[ht]
    \centering
    \includegraphics[width=\textwidth]{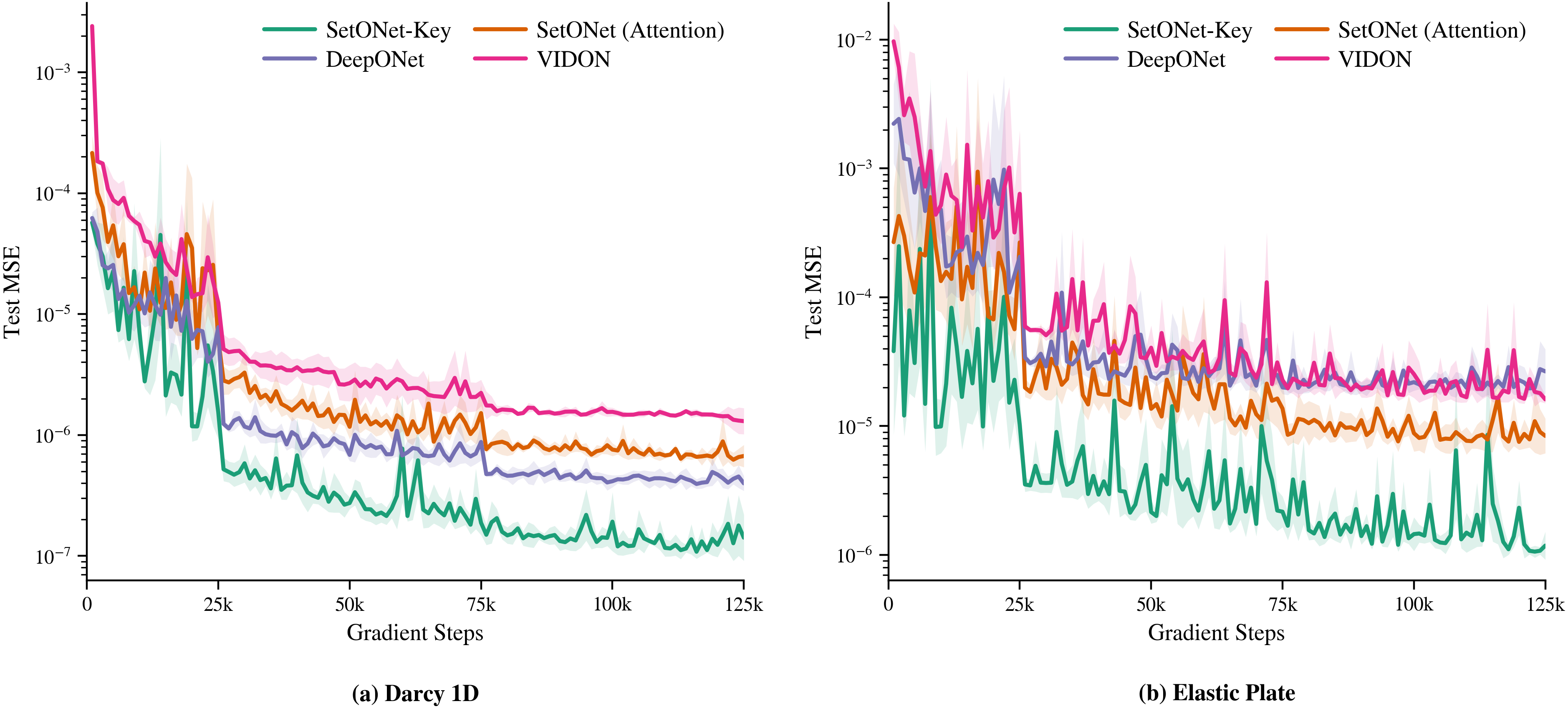}
    \caption{Test MSE histories on two classical PDE benchmarks: (a) Darcy 1D and (b) Elastic Plate. Solid curves show the mean over five independent random initializations, and shaded bands indicate one standard deviation across seeds in $\log_{10}$-space. Curves are reported without temporal smoothing.}
    \label{fig:darcy_elastic_loss}
\end{figure}

We test our approach on a commonly used 2D benchmark problem that describes an elasticity model; variations of the same problem have been used in \cite{ingebrand2025basis, goswami2022deep}.

In this example, we consider a rectangular plate with a center hole subjected to in-plane loading modeled as a two-dimensional problem of plane stress elasticity. The problem is governed by the  equation
\begin{equation}
\label{eq:elasticity}
    \nabla \cdot \boldsymbol{\sigma} + \boldsymbol{f}(\boldsymbol x) = \mathbf{0},\hspace{10pt} \boldsymbol x = (x,y),
\end{equation}
with boundary conditions 
\begin{equation*}
\label{eq:elasticity-BC}
    (u, v) = 0, \hspace{10pt} \forall\;\; x=0,
\end{equation*}
where $\boldsymbol{\sigma}$ is the Cauchy stress tensor, $\boldsymbol{f}$ is the body force, $u(\boldsymbol x)$ and $v(\boldsymbol x)$ represent the $x$- and $y$-displacement, respectively. Additionally, we note that $E$ and $\nu$ represent the material’s Young’s modulus and Poisson’s ratio, respectively, and are fixed for the problem. The relation between stress and displacement in plane stress conditions is defined as:
\begin{equation}
    \begin{bmatrix}
        \sigma_{xx}\\ \sigma_{yy}\\ \tau_{xy}
    \end{bmatrix} = \frac{E}{1-\nu^2}
    \begin{bmatrix}
        1 & \nu & 0 \\ \nu & 1 & 0 \\ 0 & 0 & \frac{1-\nu}{2}
    \end{bmatrix} \times     
    \begin{bmatrix}
        \frac{\partial u}{\partial x} \\ \frac{\partial v}{\partial y} \\ \frac{\partial u}{\partial y} + \frac{\partial v}{\partial x}
    \end{bmatrix}.
\end{equation}

In this example, we model the loading conditions $f(\boldsymbol x)$ applied to the right edge of the plate as a Gaussian random field. Our goal for the task is to learn the mapping $\mathcal{T}: f(\boldsymbol x)\rightarrow u(\boldsymbol x)$ from the random boundary load to the horizontal displacement field. We use the modified dataset for this problem provided in \cite{ingebrand2025basis}.

\begin{figure}[ht]
    \centering
    \includegraphics[width=\textwidth]{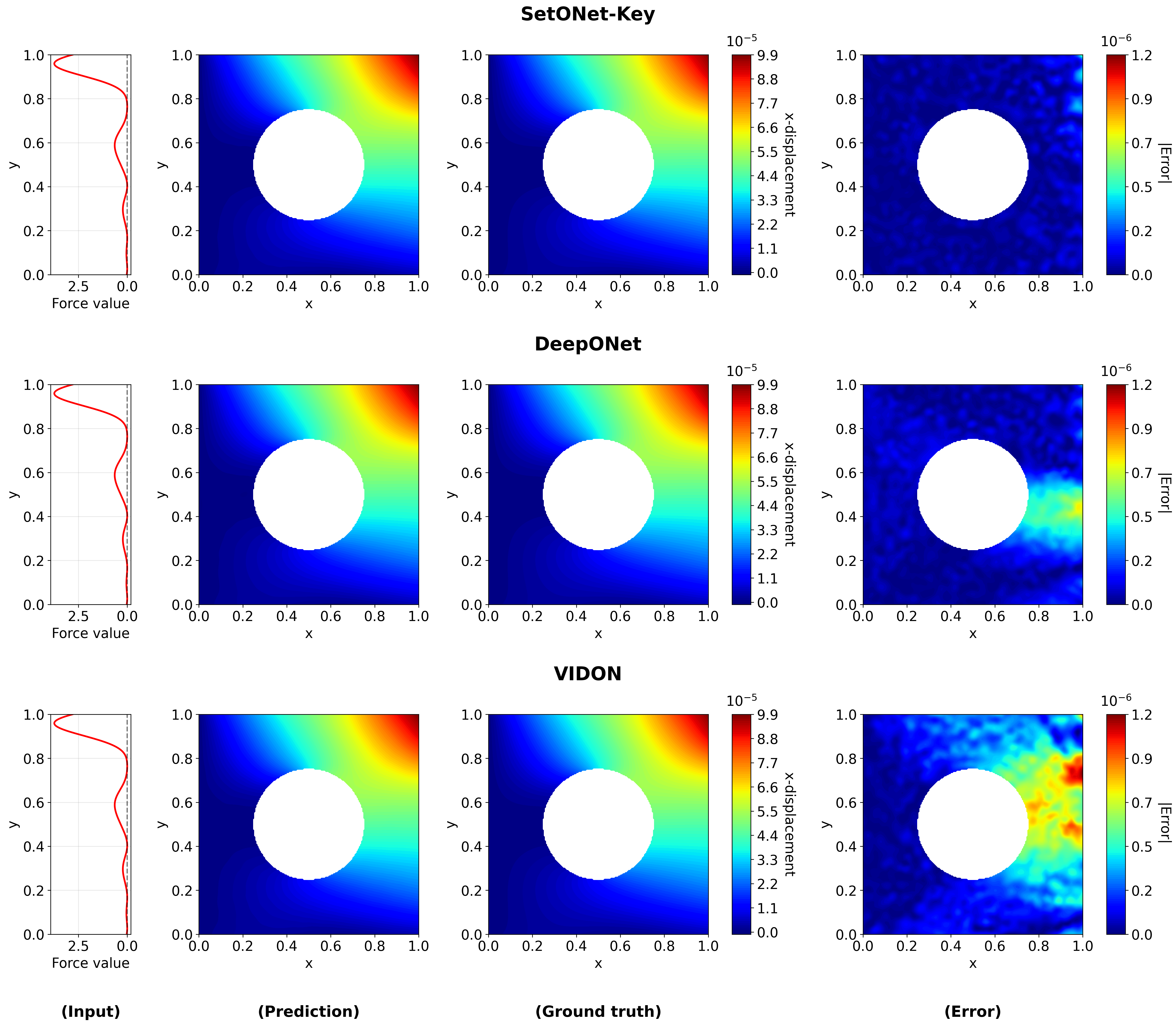}
    \caption{Qualitative comparison for the Elastic Plate benchmark. Rows correspond to SetONet-Key, DeepONet, and VIDON; columns show the input forcing function, predicted \(x\)-displacement field, ground-truth \(x\)-displacement field, and absolute error \((|\mathrm{pred}-\mathrm{true}|)\). The circular mask indicates the hole in the plate.}
    \label{fig:Elastic_SetONet_prediction}
\end{figure}

We display the validation curve in~\cref{fig:darcy_elastic_loss} and the relative error results in \cref{tab:setonet_results_pow}. Notably, for this example we observe that SetONet outperforms standard DeepONet and VIDON across all sensor configurations in both mean $\ell_2$ relative error and variance. It is also worth pointing out that, similar to other examples, while SetONet performs best under the fixed-sensor setting as expected, the performance drop is minimal when switching to variable sensor configurations or when tested with random sensor drop-off.
We display prediction results of all the trained moels in \cref{fig:Elastic_SetONet_prediction}. Here we only show the $x$-displacement. 

In summary, SetONet functions as a drop-in replacement for DeepONet: it achieves lower error under fixed sensor layouts and remains stable when layouts vary or sensors are dropped at evaluation. By operating directly on unordered location–value pairs, it removes the fixed-grid assumption and provides a more \emph{universal}, \emph{robust}, and \emph{general} formulation for operator learning—while preserving the same simple branch–trunk architecture and the same number of trainable parameters as the baseline. Notably, comparing to parameter-heavy architectures such as VIDON, SetONet achieves better performances while having significantly smaller models sizes, making it an appealing choice when training efficiency is needed.

Having established that SetONet achieves lower error than the baseline methods on standard PDE operator learning benchmarks, we now turn to problems where input functions are fundamentally unstructured. These inputs can be discontinuous or may not admit explicit formulas, making them impractical to approximate over a fixed mesh. In the following examples, we demonstrate that SetONet provides a natural and efficient way to handle such problems.

\subsection{Heat Conduction with Multiple Sources}
\label{sec: heat}

\begin{figure}[H]
    \centering
    \includegraphics[width=\textwidth]{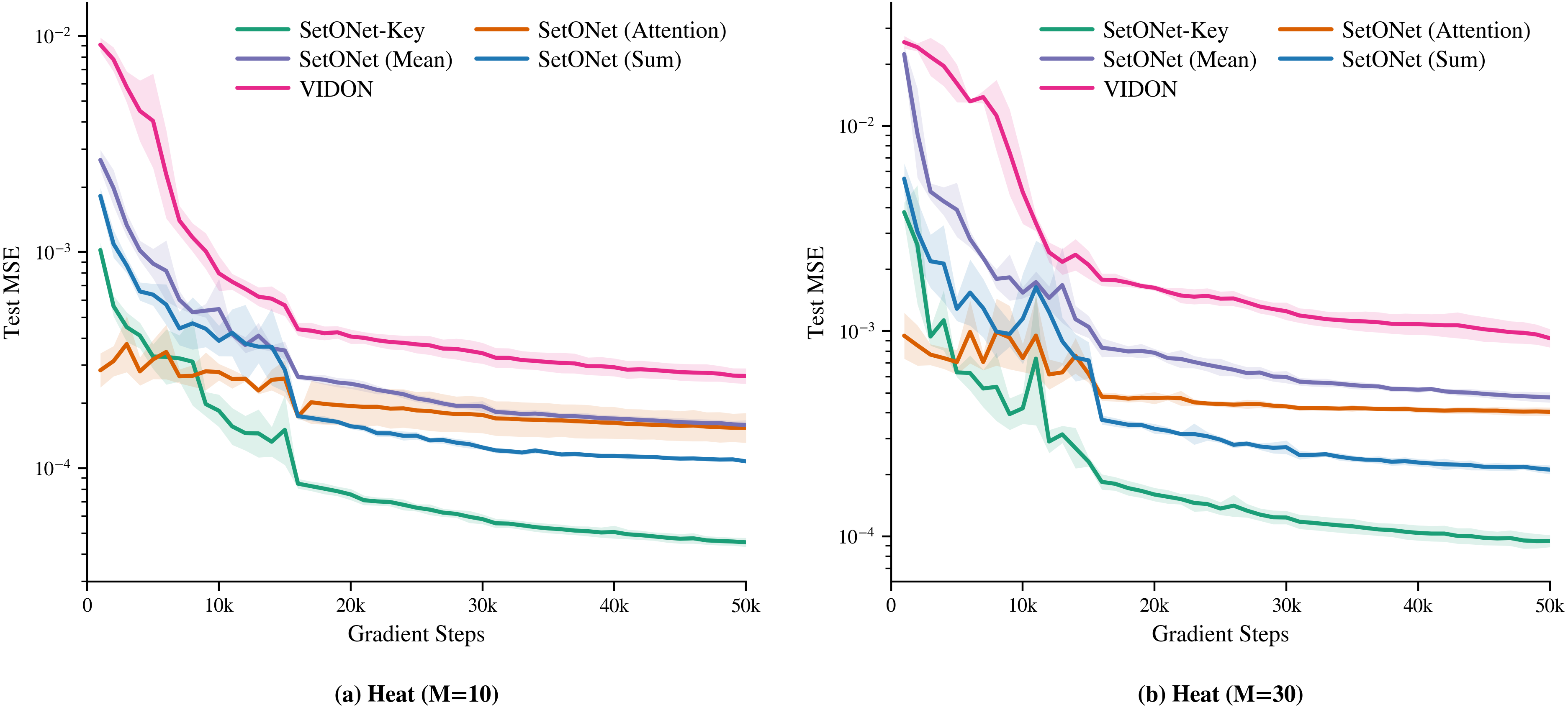}
    \caption{Test MSE histories for the Heat benchmark with (a) $M{=}10$ and (b) $M{=}30$ source points. Solid curves show the mean over five independent random initializations, and shaded bands indicate one standard deviation across seeds in $\log_{10}$-space. Curves are reported without temporal smoothing.}
    \label{fig:heat_loss_main}
\end{figure}

We consider the steady-state heat conduction problem on the unbounded plane 
$(x,y)\in\mathbb{R}^2$. Given $M$ point heat sources of strengths $s_i$ at 
positions $(x_i,y_i)$, the governing equation is
\begin{equation}
\Delta u(x,y) \;=\; \sum_{i=1}^M s_i \,\delta\big(x-x_i,\,y-y_i\big),
\end{equation}
where $u(x,y)$ denotes the temperature field and $\delta$ is the two-dimensional 
Dirac delta function. A corresponding analytic solution is given by the Green's function~\cite{courant2024methods,evans2022partial}
\begin{equation}
u(x,y) \;=\; \sum_{i=1}^M \frac{s_i}{2\pi} \,
\log\!\Big(\sqrt{(x-x_i)^2+(y-y_i)^2+\epsilon^2}\,\Big),
\end{equation}
where $\epsilon>0$ is a small softening parameter introduced to avoid 
singularities at the source locations.

\begin{table}[htbp]
\centering
\caption{Test relative $\ell_2$ error (mean $\pm$ std) comparing SetONet variants and VIDON on point-cloud problems. For Heat, $M$ denotes the number of heat-source points.}
\label{tab:pointcloud_results}
\setlength{\tabcolsep}{6pt}
\renewcommand{\arraystretch}{1.12}
\resizebox{\linewidth}{!}{%
\begin{tabular}{@{}ccccc@{}}
\toprule
\makecell{\textbf{Heat} \\ ($M{=}10$)}
& \makecell{\textbf{Heat} \\ ($M{=}30$)}
& \makecell{\textbf{Advection} \\ \textbf{Diffusion}}
& \textbf{Diffraction}
& \makecell{\textbf{Optimal} \\ \textbf{Transport}} \\
\midrule
\multicolumn{5}{c}{\textbf{SetONet-Key}} \\
\addlinespace[5pt]
\cellcolor{gray!20}\sci{1.35}{-2} $\pm$ \sci{3.75}{-4} & \cellcolor{gray!20}\sci{7.01}{-3} $\pm$ \sci{2.48}{-4} & \cellcolor{gray!20}\sci{4.47}{-2} $\pm$ \sci{8.01}{-4} & \cellcolor{gray!20}\sci{4.22}{-2} $\pm$ \sci{1.95}{-3} & \sci{5.49}{-2} $\pm$ \sci{1.56}{-3} \\
\addlinespace[10pt]
\multicolumn{5}{c}{\textbf{SetONet (Attention)}} \\
\addlinespace[5pt]
\sci{2.47}{-2} $\pm$ \sci{2.19}{-3} & \sci{1.43}{-2} $\pm$ \sci{4.38}{-4} & \sci{7.58}{-2} $\pm$ \sci{2.06}{-3} & \sci{8.23}{-2} $\pm$ \sci{3.16}{-3} & \cellcolor{gray!20}\sci{5.21}{-2} $\pm$ \sci{2.96}{-4} \\
\addlinespace[10pt]
\multicolumn{5}{c}{\textbf{SetONet (Mean)}} \\
\addlinespace[5pt]
\sci{2.51}{-2} $\pm$ \sci{6.21}{-4} & \sci{1.57}{-2} $\pm$ \sci{6.02}{-4} & \sci{7.34}{-2} $\pm$ \sci{1.20}{-3} & \sci{7.24}{-2} $\pm$ \sci{1.48}{-3} & \sci{5.39}{-2} $\pm$ \sci{1.62}{-4} \\
\addlinespace[10pt]
\multicolumn{5}{c}{\textbf{SetONet (Sum)}} \\
\addlinespace[5pt]
\sci{2.09}{-2} $\pm$ \sci{1.00}{-4} & \sci{1.03}{-2} $\pm$ \sci{2.70}{-4} & \sci{5.70}{-2} $\pm$ \sci{4.18}{-4} & \sci{6.60}{-2} $\pm$ \sci{1.30}{-3} & \sci{5.42}{-2} $\pm$ \sci{2.03}{-4} \\
\addlinespace[10pt]
\multicolumn{5}{c}{\textbf{VIDON}} \\
\addlinespace[5pt]
\sci{3.27}{-2} $\pm$ \sci{1.63}{-3} & \sci{2.20}{-2} $\pm$ \sci{1.18}{-3} & \sci{8.23}{-2} $\pm$ \sci{1.66}{-3} & \sci{1.01}{-1} $\pm$ \sci{4.47}{-3} & \sci{5.41}{-2} $\pm$ \sci{9.01}{-4} \\
\bottomrule
\end{tabular}%
}
\end{table}

\begin{figure}[ht]
    \centering
    \includegraphics[width=\textwidth]{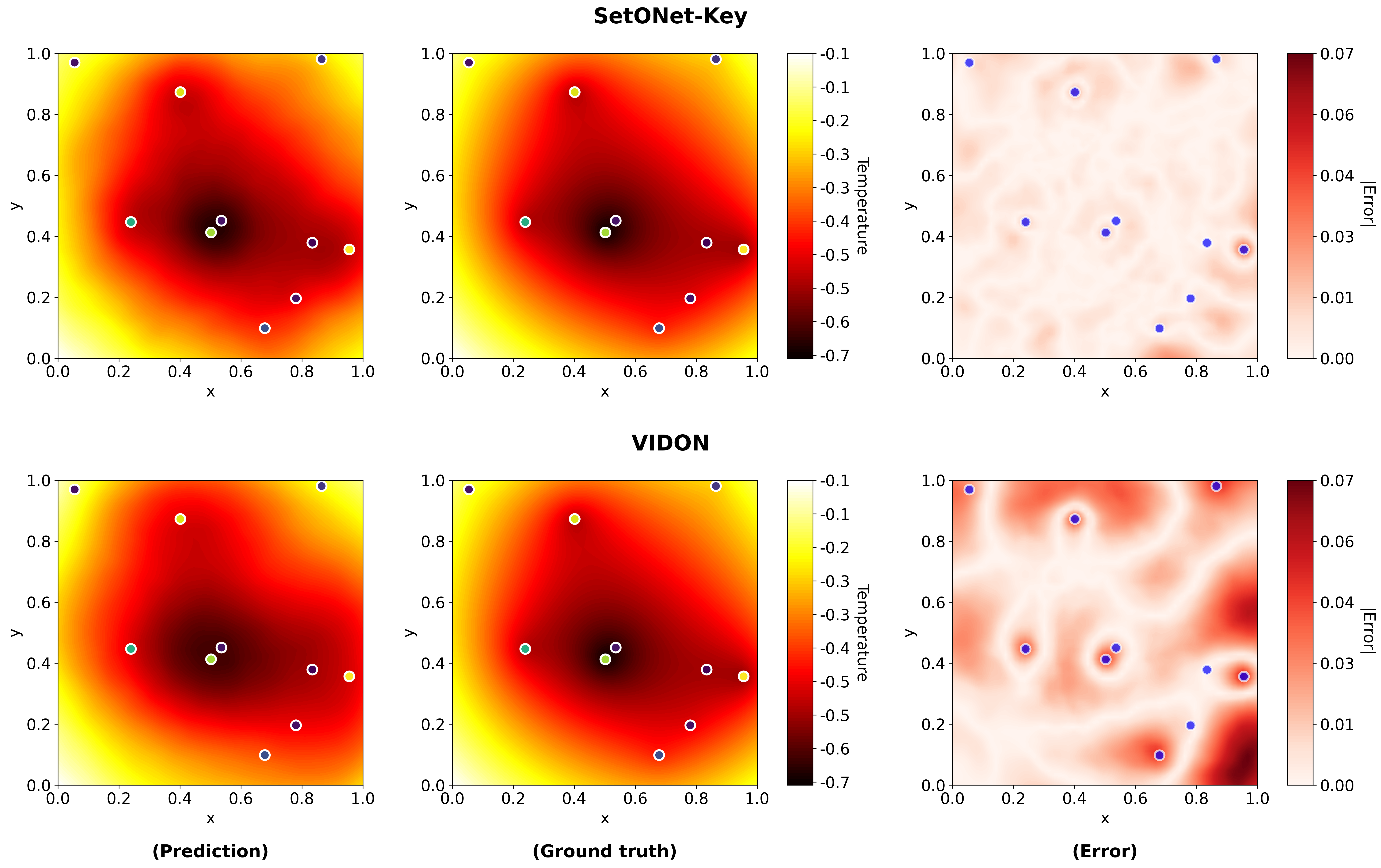}
    \caption{Qualitative comparison for the 2D Heat Conduction benchmark with 10 source points. Rows correspond to SetONet-Key and VIDON; columns show the predicted temperature field, ground-truth temperature field, and absolute error \((|\mathrm{pred}-\mathrm{true}|)\). Circular markers denote source locations; in the prediction and ground-truth panels, marker color indicates source power.}
    \label{fig:Heat_SetONet_prediction_N10}
\end{figure}

\begin{figure}[ht]
    \centering
    \includegraphics[width=\textwidth]{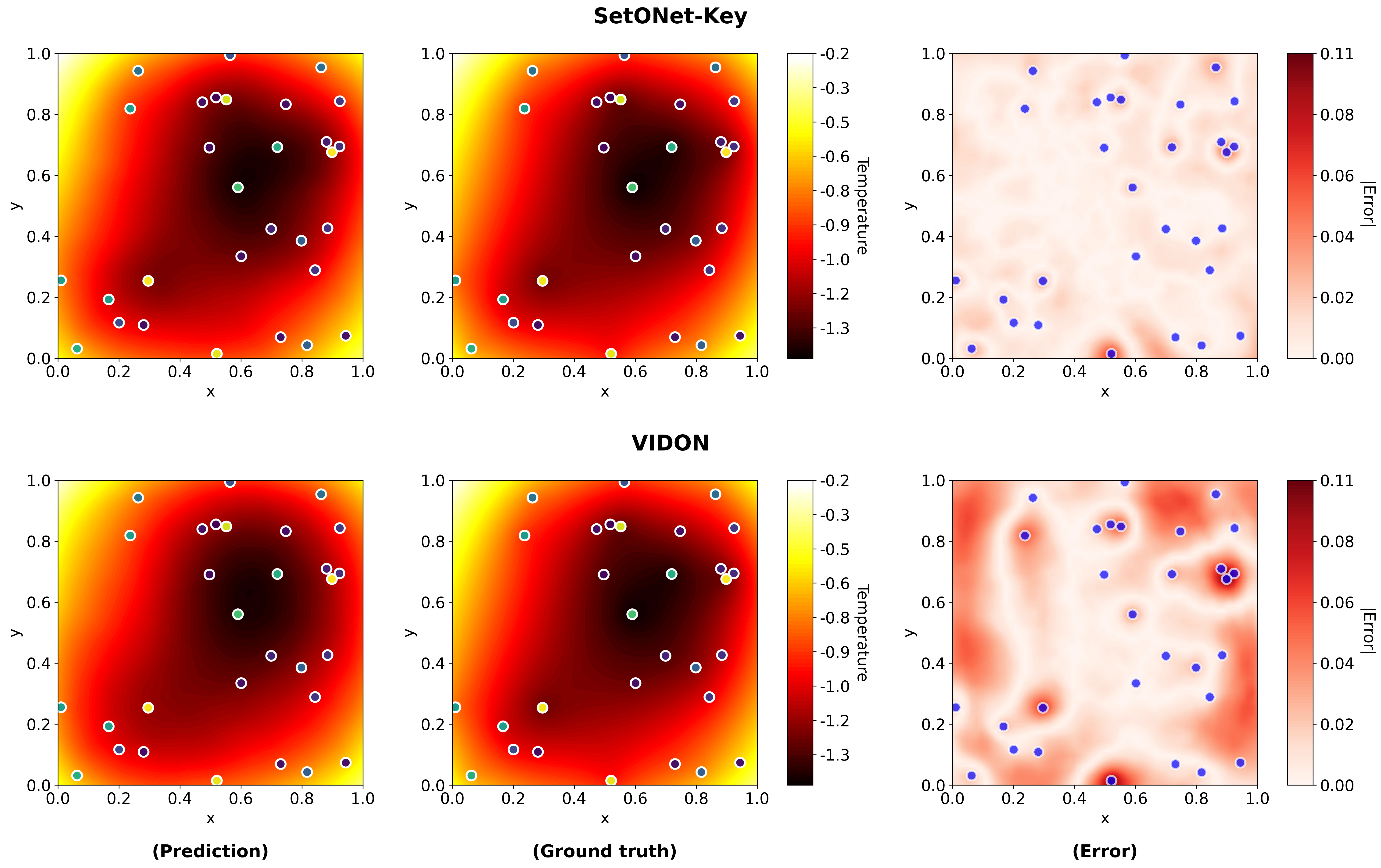}
    \caption{Qualitative comparison for the 2D Heat Conduction benchmark with 30 source points. Rows correspond to SetONet-Key and VIDON; columns show the predicted temperature field, ground-truth temperature field, and absolute error \((|\mathrm{pred}-\mathrm{true}|)\). Circular markers denote source locations; in the prediction and ground-truth panels, marker color indicates source power.}
    \label{fig:Heat_SetONet_prediction_N30}
\end{figure}

For this example, we consider two different subproblems, namely where we have $M = 10$ and $M = 30$ randomly sampled heat sources. Each source is also assigned different values, increasing the difficulty of the problems. We show the aggregated results in~\cref{tab:pointcloud_results}. For both subproblems we achieve $\ell_2$ relative error less than $2\%$. Notably, we find that increasing the number of sources in this case does not lead to an increase in error, suggesting that our proposed approach is scalable and not limited to small sample sizes. Our method is also robust, with variance between different random trials at the $10^{-4}$ level. Lastly, we present prediction results and corresponding ground truth solutions in \cref{fig:Heat_SetONet_prediction_N10} and \cref{fig:Heat_SetONet_prediction_N30}. In both cases, the predictions match the true solutions closely. During training, we evaluate the solution on an adaptive grid to accurately capture the regions with sharp gradients.

\subsection{Advection-Diffusion Equation for Modeling Chemical Concentration}
We consider the steady two-dimensional advection--diffusion equation with
constant diffusivity $d > 0$ and uniform velocity field 
$\mathbf v \in \mathbb R^2$. For point sources of strength $s_i$ located 
at positions $\boldsymbol{x}_i = (x_i, y_i)$, the governing PDE is
\begin{equation}
- d \,\Delta u(\boldsymbol{x}) + \mathbf v \cdot \nabla u(\boldsymbol{x}) 
= \sum_{i=1}^M s_i \,\delta(\boldsymbol{x} - \boldsymbol{x}_i), 
\quad \boldsymbol{x} \in \mathbb R^2.
\end{equation}
On the unbounded plane, the fundamental solution that decays upstream 
and laterally is given by
\begin{equation}
g(\boldsymbol{r} ) = \frac{1}{2\pi d}
\exp\!\left(\tfrac{\mathbf v \cdot \boldsymbol{r}}{2d}\right) 
K_0\!\left(\tfrac{\|\mathbf v\| \,\|\boldsymbol{r}\|}{2d}\right),
\end{equation}
where $\boldsymbol{r} = \boldsymbol{x} - \boldsymbol{x}_i$ and $K_0$ denotes the 
modified Bessel function of the second kind. By linearity, the total 
concentration field is
\begin{equation}
u(\boldsymbol{x}) = \sum_{i=1}^M s_i \, g(\boldsymbol{x} - \boldsymbol{x}_i).
\end{equation}
In the limit $\mathbf v \to 0$, this expression reduces (up to an additive 
constant) to the two-dimensional diffusion Green’s function 
$-\tfrac{1}{2\pi d}\log\|\boldsymbol{r}\|$ used in \cref{sec: heat}.
This formula provides a closed-form expression for the concentration field at any point in space given multiple chemical sources.

\begin{figure}[ht]
  \centering
  \includegraphics[width=\textwidth]{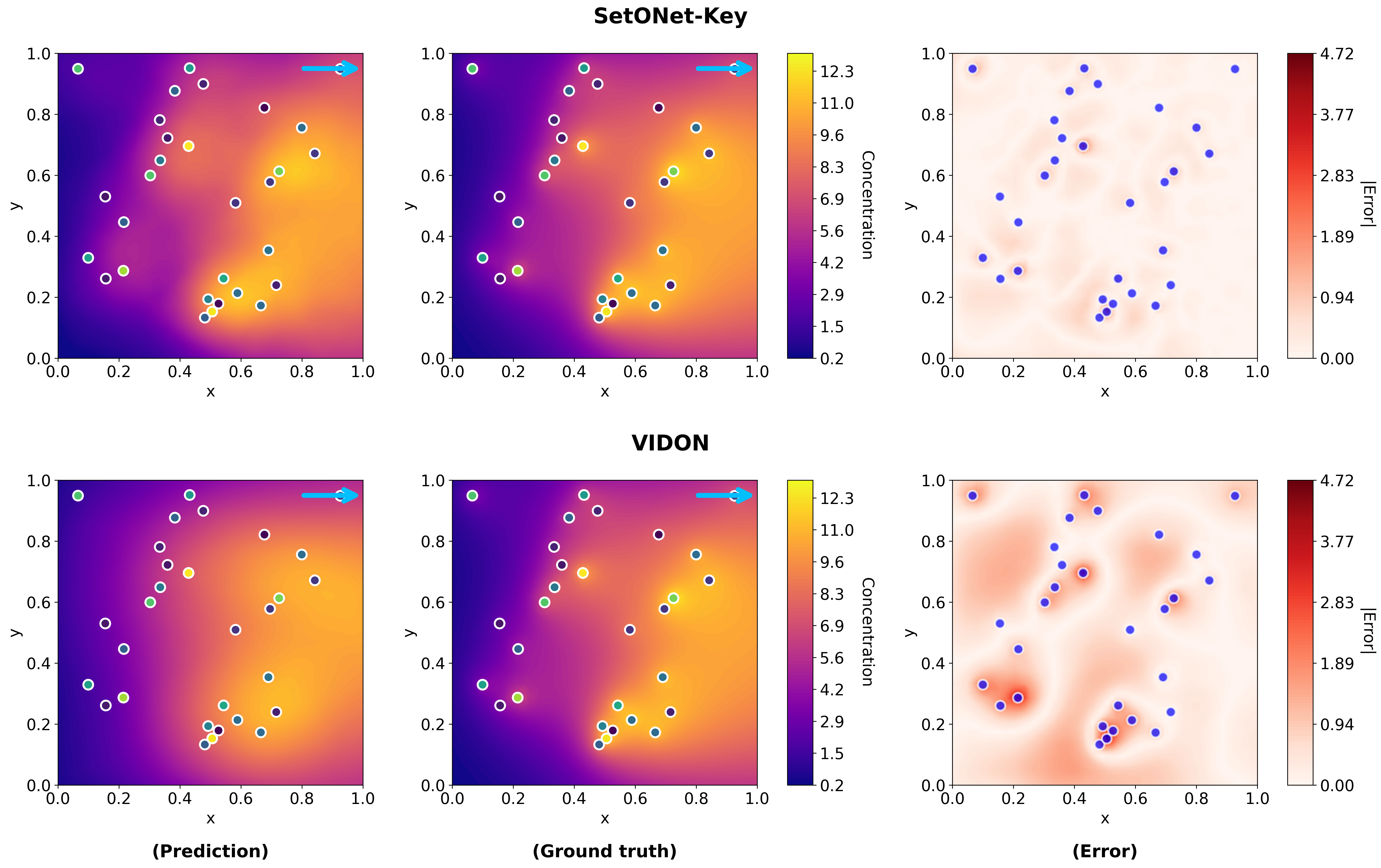}
  \caption{Qualitative comparison for the 2D Advection--Diffusion benchmark. Rows correspond to SetONet-Key and VIDON; columns show the predicted concentration field, the ground-truth concentration field, and the absolute error \((|\mathrm{pred}-\mathrm{true}|)\). Circular markers denote source locations; in the prediction and ground-truth panels, marker color indicates source emission rate. The cyan arrow indicates the advection direction.}
  \label{fig:Concentration_SetONet_prediction}
\end{figure}

\begin{figure}[ht]
    \centering
    \includegraphics[width=\textwidth]{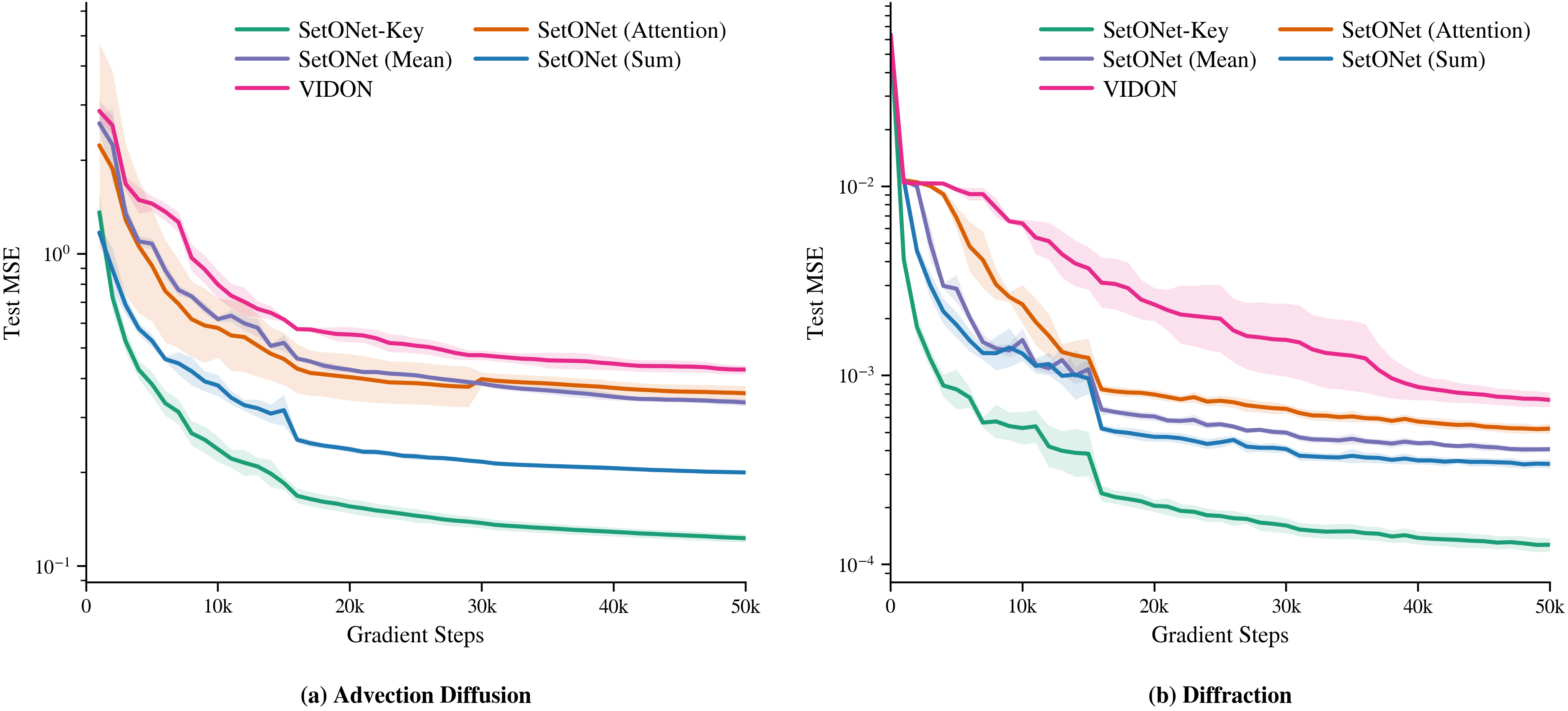}
    \caption{Test MSE histories on two point-cloud operator-learning benchmarks: (a) Advection Diffusion and (b) Diffraction. Solid curves show the mean over five independent random initializations, and shaded bands indicate one standard deviation across seeds in $\log_{10}$-space. Curves are reported without temporal smoothing.}
    \label{fig:concentration_diffraction_loss}
\end{figure}

For this benchmark we use \(M{=}30\) source points with log-uniform emission rates under a uniform wind \(\mathbf{v}{=} \begin{bmatrix}1,0
\end{bmatrix}^\top \). Aggregated results in \cref{tab:pointcloud_results} show a relative \(\ell_2\) error of $4.49\%$ with low variance at $0.08\%$. Qualitatively, SetONet reproduces the anisotropic plume structure—downwind elongation, lateral decay, and upstream attenuation—yielding close agreement with the ground truth, see~(\cref{fig:Concentration_SetONet_prediction}). These results indicate that SetONet handles point-source superposition and advection-induced anisotropy reliably within a single, end-to-end model. 
Similar to the heat conduction example, we use an adaptive grid to evaluate the solutions in the regions with sharp gradients to ensure accuracy.

\subsection{Phase-Screen Diffraction}
\label{subsec:results-diffraction}
We introduce a two-dimensional \emph{Phase-Screen Diffraction} benchmark in which an unordered set of localized phase perturbations defines the initial condition of a complex wavefield, and the learning target is the propagated field at a fixed time.

We consider the periodic domain $\Omega := [0,1)^2 \cong \mathbb{T}^2$ and a fixed propagation time $t_0>0$. Each input to the operator is modeled as the set of $M$ Gaussian phase perturbations
\[
\mathcal{X} := \big\{(\boldsymbol{x}_i,\alpha_i,\ell_i)\big\}_{i=1}^{M},
\qquad
\boldsymbol{x}_i \in \Omega,\ \alpha_i\in\mathbb{R},\ \ell_i>0 .
\]
Because the domain is periodic, these collectively define a phase screen
\begin{equation}
\Phi(\boldsymbol{x};\mathcal{X})
=
\sum_{i=1}^{M}
\alpha_i
\exp\!\left(
-\frac{\|\boldsymbol{x}-\boldsymbol{x}_i\|_{\mathbb{T}}^2}{2\ell_i^2}
\right).
\label{eq:phase_screen}
\end{equation}
where $\|\boldsymbol{x}-\boldsymbol{x}_i\|_{\mathbb{T}}^2 := \sum_{j=1}^{2} \operatorname{wrap}(x_j-x_{i,j})^2$ and $\operatorname{wrap}(s):=s-\Big\lfloor s+\tfrac{1}{2}\Big\rfloor \in \big[-\tfrac{1}{2},\tfrac{1}{2}\big)$.

The complex initial wavefield is defined as
\begin{equation}
u(\boldsymbol{x}, 0;\mathcal{X}) = A(\boldsymbol{x})\,
\exp\!\big(i\,\Phi(\boldsymbol{x};\mathcal{X})\big),
\label{eq:initial_field}
\end{equation}
where
\[
A(\boldsymbol{x})
=
\exp\!\left(
-\frac{\|\boldsymbol{x}-(\tfrac12,\tfrac12)^\top\|_2^2}{2\sigma_{\mathrm{env}}^2}
\right),
\]
and $\sigma_{\mathrm{env}}>0$ is the fixed Gaussian envelope width.

The field evolves according to the free Schr\"odinger equation
\begin{equation}
i\,\partial_t u(\boldsymbol{x},t)
=
-\tfrac12 \Delta u(\boldsymbol{x},t),
\qquad
(\boldsymbol{x},t)\in\Omega\times(0,t_0],
\label{eq:schrodinger_diffraction}
\end{equation}
with periodic boundary conditions in both spatial directions and initial condition
$u(\boldsymbol{x},0)=u(\boldsymbol{x}, 0;\mathcal{X})$.
We define the operator learning objective as the mapping $\mathcal{T}:
\mathcal{X} \;\longmapsto\; u(\cdot,t_0)$.

On $\Omega=\mathbb{T}^2$, the Schr\"odinger equation diagonalizes in the Fourier series basis. Let $\mathcal{F}$ denote the Fourier series transform on $\Omega$ with wavevectors $\boldsymbol{\xi}\in(2\pi\mathbb{Z})^2$.
The solution at time $t_0$ is obtained spectrally as
\begin{equation}
u(\cdot,t_0)
=
\mathcal{F}^{-1}\!\left[
\exp\!\left(-\tfrac{i}{2}\|\boldsymbol{\xi}\|_2^2 t_0\right)
\mathcal{F}\big[u(\cdot, 0;\mathcal{X})\big]
\right],
\label{eq:fourier_propagation}
\end{equation}
which we evaluate numerically using a Fourier spectral method (FFT) on a uniform grid.

For data generation, we evaluate the solution on a uniform $128 \times 128$ grid on $\Omega$ at fixed propagation time $t_0=0.1$.
Each sample contains $M=10$ phase bumps with
$\boldsymbol{x}_i\sim\mathrm{Uniform}(\Omega)$,
$\alpha_i\sim\mathrm{Uniform}(-\pi/2,\pi/2)$,
and fixed width $\ell_i=0.4$.
We generate a total of $20{,}000$ training samples and $1{,}000$ test samples. The complex field outputs are stored as two separate channels corresponding to the real and imaginary parts.

We note that, despite similarities in notation, this benchmark differs fundamentally from the point-source benchmarks (Heat and Advection--Diffusion), where the solution is a linear superposition of Green's function responses from individual sources. Here the perturbations combine nonlinearly through the phase term $\exp\!\big(i\,\Phi(\boldsymbol{x};\mathcal{X})\big)$ before propagation, making the resulting solution operator fundamentally different and much more challenging.

Quantitative results for this benchmark are reported in~\cref{tab:pointcloud_results}, where all SetONet variants outperform VIDON and exhibit lower variance across seeded runs. Qualitative predictions are shown in~\cref{fig:Diffraction_SetONet_prediction}, where we visualize both the real and imaginary parts of the complex field.

\begin{figure}[ht]
    \centering
    \includegraphics[width=\textwidth]{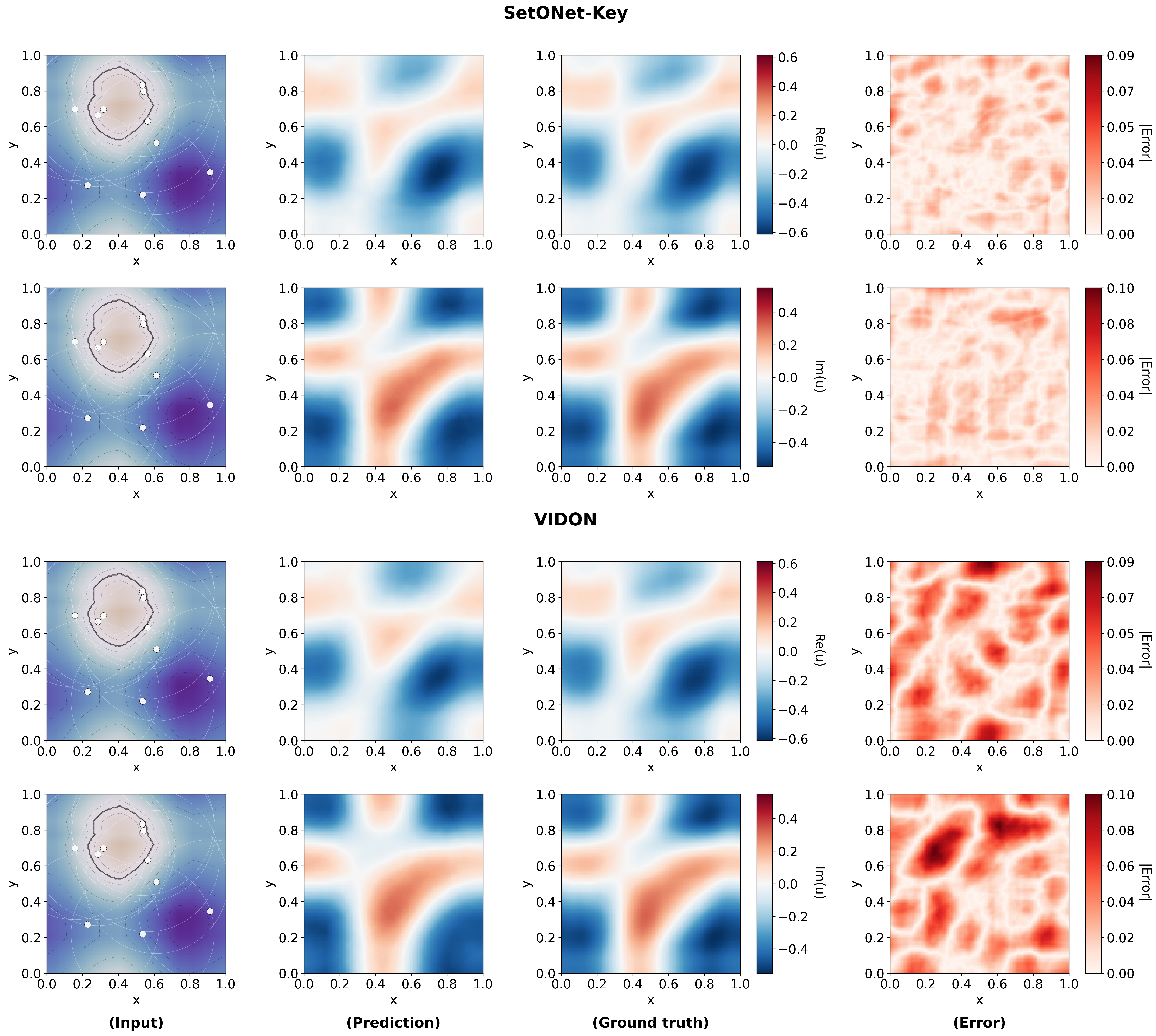}
    \caption{Qualitative comparison for the 2D Phase-Screen Diffraction benchmark. From top to bottom, rows correspond to SetONet-Key (\(\mathrm{Re}(u)\)), SetONet-Key (\(\mathrm{Im}(u)\)), VIDON (\(\mathrm{Re}(u)\)), and VIDON (\(\mathrm{Im}(u)\)). Columns show the input phase screen, predicted field component, ground-truth component, and absolute error \((|\mathrm{pred}-\mathrm{true}|)\).}
    \label{fig:Diffraction_SetONet_prediction}
\end{figure}

\subsection{Optimal Transport Problem}
To further demonstrate the versatility of our proposed approach, we consider the following optimal transport problem with quadratic cost~\cite{santambrogio2015optimal}. Given an initial density 
$\rho_0$ and a target density $\rho_1$ defined on $\Omega = [-5, 5]^2$ with equal total mass, the \emph{Monge}  formulation seeks a measurable map $u:\Omega\to\Omega$ minimizing
\begin{equation}
\int_\Omega \|\boldsymbol{x} - u(\boldsymbol{x})\|_2^2 \, \rho_0(\boldsymbol{x}) \, d\boldsymbol{x},
\end{equation}
subject to the pushforward constraint $u_{\#} \rho_0 = \rho_1$. To solve the problem more easily, we also consider the \emph{Kantorovich} formulation, which reads
\begin{gather}
\min_{\pi \ge 0} \int_{\Omega\times\Omega} \|\boldsymbol{x} - \boldsymbol{y}\|_2^2 \,\pi(\boldsymbol{x},\boldsymbol{y})\,d\boldsymbol{x}\,d\boldsymbol{y},\\
\text{subject to} \quad 
\int_{\Omega} \pi(\boldsymbol{x},\boldsymbol{y})\,d\boldsymbol{y} = \rho_0(\boldsymbol{x}), 
\quad 
\int_{\Omega} \pi(\boldsymbol{x},\boldsymbol{y})\,d\boldsymbol{x} = \rho_1(\boldsymbol{y}).
\end{gather}
This allows us to solve the discrete Kantorovich form with entropic regularization using a Sinkhorn algorithm~\cite{cuturi2013sinkhorn}. The resulting optimal coupling $P$ can then be converted to a barycentric map for $u$. For any sample point from $\rho_0$ the point-wise transport map can be obtained through interpolation.

The operator learning problem is formulated as finding the mapping $\mathcal{T} :\rho_0(\boldsymbol x)\rightarrow u(\boldsymbol{x})$; in practice the true density function $\rho_0(\boldsymbol{x})$ is not always accessible, and our novel approach allows it to be replaced by drawn samples $\{\boldsymbol{x}_i\}_{i=1}^{M}\sim\rho_0$, which is not directly supported by standard fixed-grid operator learning methods. The source density $\rho_0$ is taken as a mixture of two Gaussians with means sampled uniformly from $[-2,2]^2$ and diagonal covariances with entries in $[0.1,1.0]$. The target density $\rho_1$ is fixed as the centered Gaussian $\mathcal{N}(\boldsymbol{0},\,0.5\,\mathbf{I}_2)$. We generate $20{,}000$ training and $1{,}000$ test examples. To form the model input, we randomly draw $512$ i.i.d.\ points from the source $\rho_0$ and represent them with unit weights. Supervision is provided at $1024$ independent query points through the displacement vectors $u(\boldsymbol{y})-\boldsymbol{y}$. The transport map is solved over an $80{\times}80$ uniform grid over the domain.

\begin{figure}[ht]
  \centering
  \includegraphics[width=\textwidth]{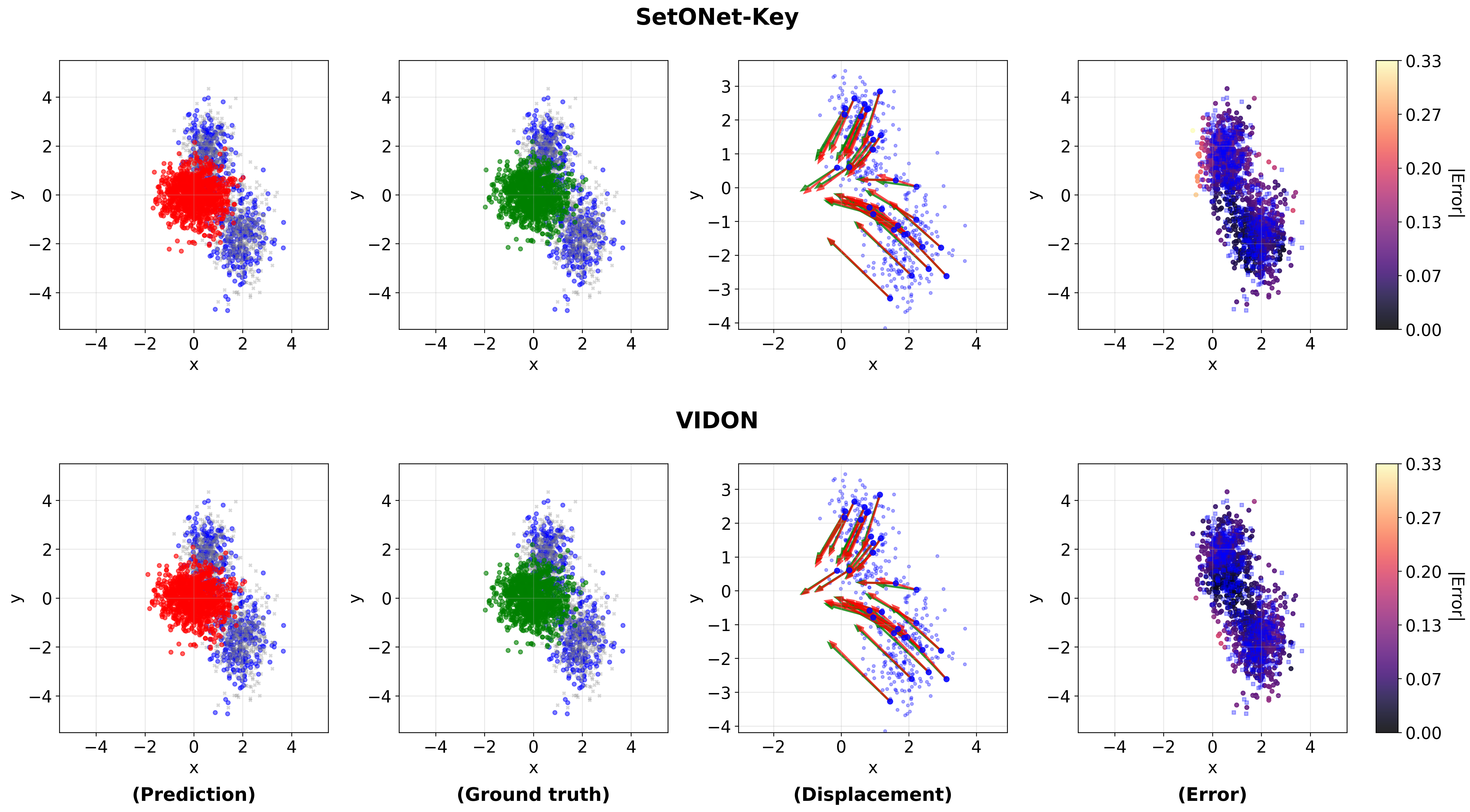}
  \caption{Qualitative comparison for the 2D optimal transport benchmark. Rows correspond to the compared models; columns show prediction, ground truth, displacement, and error. In the prediction and ground-truth columns, source points are shown in blue, independent query points in gray (\(x\)-markers), and transported query locations in red (prediction) or green (ground truth). The displacement column overlays predicted (red) and ground-truth (green) displacement vectors on sampled source points (blue). The error column reports the pointwise displacement error magnitude \(\|\mathrm{pred}-\mathrm{true}\|\) at query points.}
  \label{fig:Transport_SetONet_prediction}
\end{figure}

Operator learning over different density functions, while common in various fields~\cite{kim2024gaussian,liao2024score}, remains understudied. One key challenge is that in many practical applications, only sample-based empirical measurements are available, making it difficult to estimate the density function over a fixed discretization. SetONet offers an alternative that directly operates on discrete samples, modeling the underlying density function implicitly and enabling efficient solving of such problems.

We present the numerical results of this problem in \cref{tab:pointcloud_results}. Here we note that while the $\ell_2$ relative error is slightly higher than in other examples, this is due to validating the transport map over a fixed domain. In practice, the solution is only important on part of the domain, depending on the underlying density. A more practical visualization is presented in \cref{fig:Transport_SetONet_prediction}, where the transport map closely matches the ground truth over point samples, demonstrating the accuracy of the learned operator. 
Furthermore, this example highlights that our proposed approach can handle point cloud inputs of various sizes and is not limited to small sample sets, making it suitable for problems at scale.

\section{Summary and Conclusion}
\label{sec:conclusion}

We introduced SetONet, a set-based operator network that removes the fixed-sensor assumption underlying standard DeepONet formulations by representing each input function as an unordered set of location--value pairs. By replacing the conventional fixed-length branch input with a permutation-invariant set encoder while preserving the branch--trunk synthesis mechanism, SetONet accommodates variable sampling patterns, missing observations, and inherently unstructured inputs within a single end-to-end architecture. Within this framework, SetONet-Key introduces geometry-aware token-based aggregation while retaining the lightweight DeepONet-style construction.

Across the classical benchmarks, broadening the branch interface from fixed sensor vectors to sets did not compromise accuracy on standard DeepONet-style tasks. In parameter-matched fixed-layout studies, SetONet-Key delivered modest but consistent improvements on the simple 1D derivative and integral operators, while the gains were more pronounced on the PDE benchmarks Darcy 1D and Elastic Plate, where the relative $\ell_2$ error was reduced by about $1.4\times$ to more than $4\times$. Under variable-layout and sensor drop-off regimes, the set-based models remained accurate and stable: SetONet-Key was the strongest variant on three of the four benchmarks, while the attention-pooled SetONet variant performed best on Derivative. Relative to the larger VIDON baseline, the strongest SetONet variants achieved substantially lower relative $\ell_2$ error in these more demanding variable-input settings, with gains commonly around $2\times$--$3\times$. Overall, these results show that the set-based formulation retains the favorable fixed-layout behavior of DeepONet while remaining robust under both the \emph{Variable} and \emph{Drop-off} settings.

Beyond the classical benchmarks, the experiments also show that the set-based formulation is effective on problems with inherently unstructured point-cloud inputs. On \emph{Heat Conduction}, \emph{Advection--Diffusion}, and \emph{Phase-Screen Diffraction}, SetONet-Key consistently outperformed the larger VIDON baseline, typically by about $1.8\times$ to $3\times$ in relative $\ell_2$ error, while also exhibiting low variance across seeded trials. On \emph{Optimal Transport}, all models performed similarly within the reported uncertainty, and no clear advantage of one architecture over another was observed. In these examples, SetONet operates directly on point sources, phase perturbations, and density samples without rasterization or multi-stage preprocessing.

More broadly, SetONet is motivated by a practical fragmentation in operator learning: inputs may be presented on regular grids, at irregular sensor locations, or as unordered observations, and these different representations often require distinct preprocessing pipelines or architectural choices. A unified branch--trunk interface makes it possible to treat these cases within a common modeling pipeline, even when input representations vary across experiments or datasets. This is particularly relevant in experimental domains such as environmental monitoring, particle tracking, and mobile sensor platforms, where measurements arise naturally as unordered samples rather than as fixed arrays.

Several directions follow naturally from the present study. One natural extension is the integration of physics-informed training objectives, enabling SetONet to learn operators in problems where measurements are sparse or hard to obtain. The set-based formulation also supports conditional learning: augmenting the available location--value pairs with additional parameters such as material properties, timestamps, or sensor quality flags would enable multi-operator training and prediction. Finally, a theoretical analysis of set-to-field operator learning, including approximation rates and generalization bounds, would help clarify the regimes in which these architectures are most effective.

\section*{Acknowledgments}
This work was funded in part by NSF 2339678, 2321040 and 2438193.  Any opinions, findings, conclusions, or recommendations expressed in this material are those of the author(s) and do not necessarily reflect the views of the funding organizations. We would like to express our heartfelt gratitude to Dr. Somdatta Goswami at Johns Hopkins University for her insightful discussions that greatly enriched this work.

\bibliographystyle{elsarticle-num}
\bibliography{bibliography}
\newpage
\appendix
\section{Model Configuration and Training Protocol}
\label{app:config}

This appendix summarizes the effective configurations used to produce the results in \cref{sec:results}. Lower-level implementation details remain in the released code and benchmark configuration files.

\subsection{Unified Architecture}
\label{app:arch}
All compared models follow a branch--trunk decomposition. The SetONet variants use the synthesis rule of \cref{eq:setonet_output_simplified} with a learnable output bias. VIDON uses the same branch--trunk contraction and retains the additional trunk basis term $\tau_0$ from the original VIDON formulation. The benchmark DeepONet baseline uses the standard branch--trunk dot product, with the branch consuming only the fixed vector of sensor values. Hyperparameters appear below.

\subsection{SetONet Variants}
\label{app:shared-setonet}
The reported SetONet results use four branch variants: SetONet (Attention), SetONet (Mean), SetONet (Sum), and SetONet-Key.
\begin{itemize}
  \item \textit{Shared trunk and positional encoding:} all SetONet variants use sinusoidal positional encoding with total dimension $d_{\mathrm{PE}}{=}64$, applied per coordinate and concatenated. All trunk networks are four-layer ReLU MLPs of width $256$.
  \item \textit{SetONet (Attention/Mean/Sum):} these three variants share the same Value Net $v_{\boldsymbol{\theta}}$ and readout map $\rho$; only the permutation-invariant aggregation changes. The shared Value Net is a ReLU MLP with hidden width $256$ and output size $32$, and the readout map $\rho$ uses hidden width $300$ on \emph{Derivative}, \emph{Integral}, \emph{Darcy 1D}, and \emph{Elastic Plate}, and hidden width $256$ on \emph{Heat}, \emph{Advection--Diffusion}, \emph{Phase-Screen Diffraction}, and \emph{Optimal Transport}. Attention pooling uses one learnable pooling token and four heads.
  \item \textit{SetONet-Key:} the reported SetONet-Key results use the separate key and value pathways of \cref{sec:setonet_architecture}, with $d_k{=}64$ and $d_v{=}32$. In implementation, the row-wise map $\mathcal{A}_{\mathrm{key}}$ is instantiated as
  \[
  \mathcal{A}_{\mathrm{key}}(\mathbf{S})_{k,i}
  =
  \frac{w_i\,a_{\mathrm{mix}}(S_{k,i})}{\sum_{j=1}^{M} w_j},
  \]
  where $w_i \ge 0$ denotes the sensor weight and $a_{\mathrm{mix}}$ is the row-wise score transform of \cref{sec:setonet_architecture,thm:setonet_key_uat}. In the reported runs, $a_{\mathrm{mix}}$ is implemented by $\mathrm{softplus}$ on the 1D benchmark families and by $\tanh$ on the 2D benchmark families, and the hidden width used in the Key Net $k_{\boldsymbol{\theta}}$ and row-wise readout $\rho_{\mathrm{tok}}$ is $300$ on \emph{Derivative}, $200$ on \emph{Integral}, \emph{Darcy 1D}, and \emph{Elastic Plate}, and $256$ on \emph{Heat}, \emph{Advection--Diffusion}, \emph{Phase-Screen Diffraction}, and \emph{Optimal Transport}. In 1D, $w_i$ are geometry-derived trapezoidal cell weights; in higher dimensions, uniform sensor weights are used. For the \emph{Derivative} family, the value map is augmented by sensor coordinates; otherwise the default value pathway depends only on sensor values.
  \item \textit{Family-level settings:} the latent size is $p{=}32$ for \emph{Derivative}, \emph{Integral}, and \emph{Darcy 1D}, and $p{=}128$ for \emph{Elastic Plate}, \emph{Heat}, \emph{Advection--Diffusion}, \emph{Phase-Screen Diffraction}, and \emph{Optimal Transport}. The positional-encoding scale is $\mathrm{PE}_{\max}=0.1$ on \emph{Derivative}, \emph{Integral}, \emph{Darcy 1D}, \emph{Elastic Plate}, and \emph{Optimal Transport}, and $\mathrm{PE}_{\max}=0.01$ on \emph{Heat}, \emph{Advection--Diffusion}, and \emph{Phase-Screen Diffraction}.
\end{itemize}

\subsection{VIDON Baseline}
\label{app:vidon}
VIDON is included on all benchmarks reported in \cref{sec:results}.
\begin{itemize}
  \item \textit{Architecture:} we use the reference multi-head VIDON architecture with $H{=}4$, encoder dimension $d_{\mathrm{enc}}{=}40$, and head output size $64$.
  \item \textit{Family-level settings:} the latent size $p$ matches the corresponding SetONet family, and the trunk includes the additional basis term associated with $\tau_0$.
\end{itemize}

\subsection{DeepONet Baseline}
\label{app:deeponet}
DeepONet is included on the fixed-sensor benchmark families for which a fixed-length branch input is meaningful, namely \emph{Derivative}, \emph{Integral}, \emph{Darcy 1D}, and \emph{Elastic Plate}.
\begin{itemize}
  \item \textit{Branch and trunk:} the branch consumes only the fixed-length sensor-value vector $\bigl(g(\boldsymbol{x}_i)\bigr)_{i=1}^M$; no sensor coordinates or positional encoding are used in the branch. The trunk acts on the raw query coordinates.
  \item \textit{Family-level settings:} the latent size $p$ matches the corresponding SetONet family. DeepONet is not used on the point-cloud benchmark family.
\end{itemize}

\subsection{Optimization and Training}
\label{app:opt}
\begin{itemize}
  \item \textit{Objective and metrics:} all models are trained with mean-squared error (MSE). We report test MSE and relative $\ell_2$ error.
  \item \textit{Optimizer:} Adam with learning rate $5\times 10^{-4}$, no weight decay, and standard gradient clipping.
  \item \textit{Training schedules:} the structured benchmark family (\emph{Derivative}, \emph{Integral}, \emph{Darcy 1D}, \emph{Elastic Plate}) uses $125$k optimization steps with batch size $64$ and effective learning-rate decays by factors $0.2$ and $0.5$ at $25$k and $75$k steps. The point-cloud benchmark family (\emph{Heat} with $M{=}10$ and $M{=}30$ source points, \emph{Advection--Diffusion}, \emph{Phase-Screen Diffraction}, \emph{Optimal Transport}) uses $50$k steps with batch size $32$ and effective decays by factors $0.2$ and $0.5$ at $15$k and $30$k steps. Some training scripts expose longer milestone lists, but only the milestones listed here are reached in the reported runs.
  \item \textit{Averaging and model selection:} all reported summary statistics are computed over five seeds ($0,1,2,3,4$). No early stopping or best-checkpoint selection is used.
\end{itemize}

\section{Dataset Details}
\label{app:datasets}

This section specifies domains, sampling protocols, PDE solvers (where applicable), and splits for all benchmarks in \cref{sec:results}. Sensor and query layouts are given below for each benchmark. 

\paragraph{Derivative and Integration Operators}
On $x\in[-1,1]$, each sample is
\[
f(x)=ax^3+bx^2+cx+e\sin x,
\qquad
a,b,c,e\sim\mathrm{Uniform}([-0.1,0.1]),
\]
with zero integration constant. For \emph{Derivative}, inputs are $M{=}100$ sensor evaluations $f(x_j)$ and targets are $f'(y_k)$ at $N_q{=}200$ uniformly spaced queries $y_k$. For the \emph{Integration} task, inputs are $f'(x_j)$ at the same sensor locations and targets are $f(y_k)$. In the fixed-layout setting, one seeded sorted random sensor set is used for all samples; in the variable-layout setting, one shared sorted random sensor set is resampled per batch. Training data are generated on the fly, and evaluation averages over $960$ independently generated test functions.

\paragraph{1D Darcy Flow}
We use the same 1D Darcy problem as in \cref{subsec:results-1ddarcy}. The forcing term is drawn as
$f(x) \sim \mathrm{GP}(0, k(x,x'))$ with squared-exponential kernel
$k(x,x')=\sigma^2 \exp\!\big(-\tfrac{(x-x')^2}{2\ell_x^2}\big)$, where $\ell_x{=}0.04$ and $\sigma^2{=}1.0$.
The PDE is discretized on a uniform $501$-point grid using a centered flux form.
Inputs use $M{=}300$ fixed sensor locations selected from the $501$-point grid by linearly spaced index sampling; targets are $u(y_k)$ at $N_q{=}300$ query locations selected in the same way.
The precomputed dataset contains $10{,}000$ training and $1{,}000$ test examples.

\paragraph{Elastic Plate}
We consider linear elasticity on the unit square with a central circular hole (radius $0.25$ at $(0.5,0.5)$) in plane stress.
Material parameters are $E=3.0\times10^7$ and $\nu=0.3$.
The left edge $x{=}0$ is clamped, $(u,v){=}(0,0)$, and a traction in the $x$-direction is applied on the right edge $x{=}1$.
Inputs are samples of the boundary load $f(\boldsymbol x)$ on the right edge at $M{=}301$ equally spaced points $\boldsymbol x_j{=}(1,y_j)$, where $y_j\in[0,1]$.
Targets are the horizontal displacement $u(\boldsymbol x_k)$ at $N_q{=}1048$ mesh nodes of the plate domain.
Fields are standardized to zero mean and unit variance using global statistics; coordinates are left unscaled.
We use $1{,}900$ training and $100$ test cases.

\paragraph{Heat Conduction with Multiple Sources}
We use the free-space Green's function of the 2D Laplacian (conductivity $k{=}1$) with softening $\epsilon{=}10^{-1}$, with source and query locations restricted to $[0,1]^2$.
For each sample, $M\in\{10,30\}$ source locations are i.i.d.\ $\mathrm{Uniform}([0,1]^2)$ and source strengths are log-uniform on $[10^{-1},1]$.
The temperature field $u$ is evaluated analytically.
For each $M$, we generate $10{,}000$ training and $1{,}000$ test samples.
Per sample, we query $N_q{=}8192$ locations via an adaptive scheme: seed a $25{\times}25$ uniform grid ($625$ points), compute a coarse field, then sample the remaining $N_q-625$ queries from a $128{\times}128$ proposal with probabilities proportional to $\exp(\beta\,m)$, where $m$ is the min--max normalized magnitude of the coarse temperature field.
We use $\beta{=}9.0$ for $M{=}10$ and $\beta{=}8.0$ for $M{=}30$.
Inputs are the source sets $\{(x_i,y_i,s_i)\}_{i=1}^{M}$; targets are $u(x,y)$ at the queried locations.

\paragraph{Advection–Diffusion}
We use the exact 2D advection–diffusion Green's function with constant diffusivity $d{=}0.1$ and uniform flow $\mathbf{v}{=}[1,0]^\top$ (wind along $+x$), with source and query locations restricted to $[0,1]^2$.
Each sample has $M{=}30$ i.i.d.\ sources in $[0,1]^2$ with emission rates drawn log-uniformly on $[10^{-1},1]$.
The concentration is evaluated in closed form with near-source regularization.
Each dataset contains $10{,}000$ training and $1{,}000$ test examples, and reported metrics average over all $1{,}000$ test samples.
Queries use $N_q{=}4096$ adaptively selected locations: seed a $25{\times}25$ uniform grid, form a coarse field, then sample the remaining points from a $128{\times}128$ proposal with probabilities proportional to $\exp(\beta\,m)$, where $m$ is the min–max normalized coarse concentration.
We use $\beta{=}4.0$.
Inputs are the source sets $\{(\boldsymbol{x}_i,s_i)\}_{i=1}^{M}$; targets are the concentrations $u(\boldsymbol{x}_k)$ at the sampled queries.

\paragraph{Phase-Screen Diffraction}
On $\Omega=[0,1)^2\cong\mathbb{T}^2$, each sample is an unordered set of $M{=}10$ phase bumps $\{(\boldsymbol{x}_i,\alpha_i,\ell_i)\}_{i=1}^{10}$ with $\boldsymbol{x}_i\sim\mathrm{Uniform}(\Omega)$, $\alpha_i\sim\mathrm{Uniform}(-\pi/2,\pi/2)$, and fixed width $\ell_i=0.4$. These parameters define a periodic minimal-image Gaussian phase field, which is combined with a Gaussian envelope of width $\sigma_{\mathrm{env}}=0.2$ and zero carrier wave number to form the initial condition. The field is then propagated to time $t_0=0.1$ under the free Schr\"odinger equation using an FFT-based spectral solver on a uniform $128{\times}128$ grid. Inputs are the bump coordinates together with the per-bump features $[\alpha_i,\ell_i]$, and targets are the propagated complex field values at all grid points, stored as separate real and imaginary channels. We use $20{,}000$ training and $1{,}000$ test samples.

\paragraph{Optimal Transport Problem}
On $\Omega=[-5,5]^2$ we construct ground truth by solving optimal transport on a uniform $80{\times}80$ grid using a Sinkhorn solver. The source density is a balanced mixture of two Gaussians with means drawn i.i.d.\ from $[-2,2]^2$ and diagonal covariances with entries in $[0.1,1.0]$; the target density is fixed as $\rho_1=\mathcal{N}(\boldsymbol{0},\,0.5\,\mathbf{I}_2)$. From the resulting coupling we form the barycentric transport map and interpolate it linearly away from the grid. For each sample, the branch input is a set of $512$ i.i.d.\ source samples from $\rho_0$, represented with unit weights. Supervision is given at $1024$ query points sampled independently of the input sensor set, with targets equal to the displacement vectors $u(\boldsymbol{y})-\boldsymbol{y}$. The corresponding $80{\times}80$ grid velocity field is also stored for visualization and auxiliary evaluation. We use $20{,}000$ training and $1{,}000$ test examples.

\end{document}